%% file: main.tex
\documentclass[11pt,letter]{article}
\usepackage[utf8]{inputenc}

\usepackage{amssymb}
\usepackage{amsmath}
\usepackage{amsthm}
\usepackage{amsfonts}
\usepackage{mathrsfs}
\usepackage[langfont=typewriter,funcfont=slant]{complexity}
\usepackage[left=1in,bottom=1in,top=1in,right=1in]{geometry}
\usepackage[colorlinks=true]{hyperref}
\hypersetup{
     colorlinks = true,
     linkcolor = teal,
     anchorcolor = teal,
     citecolor = teal,
     filecolor = teal,
     urlcolor = teal
     }

\usepackage{framed}
\usepackage{soul}
\usepackage{thmtools,thm-restate}
\usepackage{csquotes}
\usepackage[nameinlink, noabbrev, capitalize]{cleveref}
\usepackage{booktabs}       %
\usepackage{nicefrac}       %
\usepackage{microtype} 

\usepackage{caption}
\usepackage{framed}

\usepackage[disable]{todonotes} %
\presetkeys{todonotes}{inline}{}

\makeatletter
\if@todonotes@disabled

\else

\fi
\makeatother

\usepackage{commath} %

\usepackage{framed}
\usepackage[ruled,noline]{algorithm2e}

\theoremstyle{definition}
\newtheorem{definition}{Definition}[section]

\theoremstyle{remark}

\theoremstyle{plain}
\newtheorem{theorem}{Theorem}[section]

\theoremstyle{plain}

\theoremstyle{plain}

\theoremstyle{plain}
\newtheorem{lemma}[theorem]{Lemma}
\newtheorem{claim}[theorem]{Claim}

\DeclareMathAlphabet\mathbfcal{OMS}{cmsy}{b}{n}
\usepackage{amsbsy}
\DeclareMathAlphabet{\mathbfscr}{OMS}{mdugm}{b}{n}

\newcommand{\br}{\mathbb{R}}

\newcommand{\ip}[2]{\left\langle #1, #2 \right\rangle}

\newcommand{\F}{\mathcal{F}}
\newcommand{\X}{\mathcal{X}}

 \renewcommand{\A}{\mathcal{A}}
\renewcommand{\D}{\mathcal{D}}

\renewcommand{\H}{\mathcal{H}}
\newcommand{\U}{\mathcal{U}}
\renewcommand{\G}{\mathcal{G}}

\newcommand{\Secref}[1]{\hyperref[#1]{Section \ref*{#1}}}
\newcommand{\Appref}[1]{\hyperref[#1]{Appendix \ref*{#1}}}

\newcommand{\vc}{\mathrm{VCDim}}
\newcommand{\ldim}{\mathrm{LDim}}

\usepackage{multirow}
\usepackage{array}

\crefname{equation}{}{}
\crefrangeformat{equation}{(#3#1#4) --~(#5#2#6)}
\crefmultiformat{equation}{(#2#1#3)}{, (#2#1#3)}{, (#2#1#3)}{, (#2#1#3)}
\crefname{lemma}{Lemma}{Lemmas}
\crefname{section}{Section}{Sections}
\crefname{subsubsubsection}{Section}{Sections}
\crefname{remark}{Remark}{Remarks}
\crefname{figure}{Figure}{Figures}
\crefname{table}{Table}{Tables}
\Crefname{lemma}{Lemma}{Lemmas}
\crefname{theorem}{Theorem}{Theorems}
\Crefname{theorem}{Theorem}{Theorems}
\crefname{proposition}{Proposition}{Propositions}
\Crefname{proposition}{Proposition}{Propositions}

\DeclareMathOperator*{\Ex}{\mathbb{E}}

\newcommand{\nadist}{\mathbfcal{D}}
\newcommand{\adist}{{\pmb{\mathscr{D}}}}

\newcommand{\vcd}{\mathrm{VCDim}\xspace}
\newcommand{\lsd}{\mathrm{LDim}\xspace}
\newcommand{\regret}{\textsc{Regret}}

\mathchardef\dash="2D

\usepackage[suppress]{color-edits}

\addauthor{nh}{magenta}
\addauthor{as}{cyan}
\addauthor{tr}{blue}

\title{\textbf{Smoothed Analysis with Adaptive Adversaries}}
\author{Nika Haghtalab \thanks{University of California, Berkeley; Email: nika@berkeley.edu} \and Tim Roughgarden \thanks{Columbia University; Email: tim.roughgarden@gmail.com} \and Abhishek Shetty \thanks{University of California, Berkeley; Email: shetty@berkeley.edu}  }
\begin{document}
    \maketitle
\allowdisplaybreaks

    \begin{abstract}
  \input{abstract}

    \end{abstract}

    \setcounter{page}{0}
    \thispagestyle{empty}
    \newpage

    \input{intro2.tex}

    \input{formulation.tex}

\input{intro_technical_overview.tex}

\input{regret.tex}

    \input{discrepancy_2.tex}

    \input{discrepancy_proof.tex}

    \input{Dispersion.tex}

\appendix
    \input{other_related_work.tex}

\input{app_claim_oblivious-bernstein-regret}

    \input{Coupling.tex}

    \input{LDProof.tex}

    \input{app_discrepanc_proofs}

    \bibliographystyle{alpha}
    \bibliography{Ref.bib}

\end{document}

%% file: abstract.tex
We prove novel algorithmic guarantees for several online problems in
the smoothed analysis model.  In this model, at each time step an
adversary chooses an input distribution with density function bounded
above pointwise by~$\tfrac{1}{\sigma}$ times that of the uniform
distribution; nature then samples an input from this distribution.
{Here, $\sigma$ is a parameter that interpolates between the extremes of
worst-case and average case analysis.}
Crucially, our results hold for {\em adaptive} adversaries that can
base their choice of an input distribution on the decisions of the
algorithm and the realizations of the inputs in the previous time
steps.  An adaptive adversary can nontrivially correlate inputs at
different time steps with each other and with the algorithm's current
state; this appears to rule out the standard proof approaches in
smoothed analysis.

This paper presents a general technique for proving smoothed
algorithmic guarantees against adaptive adversaries, in
effect reducing the setting of an adaptive adversary to the much
simpler case of an oblivious adversary (i.e., an adversary that
commits in advance to the entire sequence of input distributions).
We apply this technique to prove strong smoothed guarantees for three
different problems:
\begin{itemize}

\item Online learning: We consider the online prediction problem, where instances are generated from an adaptive sequence of $\sigma$-smooth distributions and the hypothesis class has VC dimension $d$. We bound the regret by $\tilde{O}\big(\sqrt{T d\ln(1/\sigma)} + d\ln(T/\sigma) \big)$ and provide a near-matching lower bound.
Our result shows that under smoothed analysis, learnability against adaptive adversaries is characterized by the finiteness of the VC dimension. This is as opposed to the worst-case analysis, where online learnability is characterized by Littlestone dimension (which is infinite even in the extremely restricted case of one-dimensional threshold functions).
{This is the most well-studied setting to which we apply our techniques. Our results fully answer an open question of~\cite{NIPS2011_4262}.}

\item Online discrepancy minimization: We consider the setting of the online Koml\'os problem, where the input is generated from an adaptive sequence of $\sigma$-smooth and isotropic distributions on the $\ell_2$ unit ball. We bound the $\ell_\infty$ norm of the discrepancy vector by $\tilde{O}\big(\ln^2\big( \frac{nT}{\sigma}\big) \big)$.
{This is as opposed to the worst-case analysis, where the tight discrepancy bound is $\Theta(\sqrt{T/n})$.}
{We show such $\mathrm{polylog}(nT/\sigma)$ discrepancy guarantees are not achievable for non-isotropic $\sigma$-smooth distributions.}

\item Dispersion in online optimization: We consider online optimization with piecewise Lipschitz functions where functions with $\ell$ discontinuities are chosen by a smoothed adaptive adversary and show that the resulting sequence is $\big( {\sigma}/{\sqrt{T\ell}}, \tilde O\big(\sqrt{T\ell} \big)\big)$-dispersed. 
{That is, every ball of radius ${\sigma}/{\sqrt{T\ell}}$ is split by $\tilde O\big(\sqrt{T\ell} \big)$ of the partitions made by these functions.
This result matches the dispersion parameters of \cite{Dispersion} for oblivious smooth adversaries, up to logarithmic factors.
On the other hand, worst-case sequences are  trivially $(0, T)$-dispersed.}

\end{itemize}

%% file: intro2.tex
\section{Introduction}

\paragraph{Smoothed analysis.}
Kryptonite for worst-case analysis comes in the form of algorithms for
which almost all inputs are ``easy'' and yet rare and pathological
inputs are ``hard.''  
Perhaps the most famous
example is the simplex method for linear programming, which
empirically always runs quickly %
but requires exponential time in the worst case (for all of the common
pivot rules)~\cite{KM72}.  Equally misleading is the worst-case exponential
running time of many popular local search algorithms, such as the
$k$-means clustering algorithm~\cite{arthur:vassilvitskii:2006} and the 2-OPT heuristic for
the traveling salesman problem (TSP)~\cite{SY91}; such behavior is literally
never observed for these algorithms in practice.\footnote{Note that in all
of these examples, the problem of constructing a hard instance is
challenging enough to justify its own research paper!}
Taken literally, worst-case analysis recommends against using the
simplex method to solve linear programs or local search as a heuristic
for the TSP, flatly contradicting decades of real-world experience.
Thus, for some important problems and algorithms, a more nuanced
analysis framework is called for.

But if not worst-case analysis, then what?  Outside of applications
with a stable and well-understood input distribution, average-case
analysis is a far too specific approach.  Spielman and Teng~\cite{ST04_2}
introduced {\em smoothed analysis}, a novel interpolation between
worst- and average-case analysis that is ideally suited for the
analysis of algorithms that almost always perform well.  In its
original formulation, an adversary chooses an arbitrary (worst-case)
input, which is then perturbed slightly by nature. Appealingly, the
framework makes no assumptions about the input other than a small
amount of uncertainty (e.g., due to measurement errors).  

In the more modern and general formulation of smoothed analysis, an
adversary directly chooses an input distribution from a family of
permissible distributions; nature then samples an input from the
adversary's distribution.
An algorithm is evaluated by its worst-case (over
the adversarially chosen input distribution) expected (over the
distribution) performance.
Performance guarantees in this model (e.g., on the expected running
time of an algorithm) are generally parameterized by the
``degree of anti-concentration'' enjoyed by the allowed input
distributions.  The holy grail in smoothed
analysis is to prove guarantees on algorithm performance that,
assuming only a low level of anti-concentration in the possible input
distributions, are far closer to average-case guarantees than
worst-case guarantees.

\paragraph{Online learning, discrepancy minimization, and
  optimization.}  Smoothed analysis makes sense for any numerical
measure of algorithm performance, but to date the vast majority of work 
on the topic
concerns the running time of algorithms for offline problems, as in
the famous examples above.
Our work here focuses
on {\em online} problems---online learning, online discrepancy
minimization, and online optimization---in which the input arrives
incrementally over time and an irrevocable decision must be made at
each time step.  Online algorithms for these problems are
traditionally assessed by their solution quality or regret (with
running time a secondary concern).
In the smoothed analysis version of these problems,
the adversary is forced to choose each piece of the input---a point
from a domain, a vector, or a function---from a distribution with
non-negligible anti-concentration.

The analysis of online algorithms traditionally distinguishes between
{\em oblivious} adversaries who choose the entire input sequence up
front (with knowledge only of the algorithm to be used) and {\em
  adaptive} adversaries that can condition each part of the input on the 
past.  In the worst-case model, this distinction is relevant only
for randomized algorithms, in which case adaptive adversaries choose
each part of the input as a function of the algorithm's previous
decisions.  When the adversary itself is forced to randomize, as in
the smoothed analysis model, the distinction between oblivious and
adaptive adversaries takes on new meaning: while an oblivious
adversary must choose a sequence of input distributions up front, an
adaptive adversary can base its current choice of an input distribution
on the decisions of the algorithm {\em and the realizations of the
  inputs} in previous time steps.
  
{Online learning, discrepancy minimization and optimization play integral roles in a wide range of fields and applications, such as algorithm design \cite{ALO2015,MWU}, game theory \cite{blum_mansour_2007,PLG}, differential privacy \cite{Equivalence,MWHR,MWDiff}, control theory \cite{RegretControl,Control2}, design of medical trials \cite{Covariates}, and robust statistics \cite{RobustNoReg}.
In these cases, adversary's adaptiveness both serves as a natural abstraction for correlations between past and presence 
 and is an essential piece of  the technical analyses (such as algorithmic reductions) that make these methods widely applicable.
}

\paragraph{The challenge of adaptive adversaries.}
A basic question is: For which online problems are adaptive
adversaries fundamentally more powerful than oblivious ones?  In the
smoothed analysis model, there is strong intuition about why a
guarantee against oblivious adversaries might not extend to, or at least
would be significantly harder to prove for, adaptive adversaries.  A
key to any smoothed analysis is, of course, to determine how to
leverage the assumed anti-concentration properties of the permissible
input distributions.  With an oblivious adversary, the input
distributions at each time step are independent of each other and of
the algorithm's current state, and the assumed anti-concentration
can typically be directly and separately exploited at each time
step.  An adaptive adversary, on the other hand, has the power to
correlate inputs at different time steps with each other and with the
algorithm's current state.  This dependence seems to rule out the
standard proof approaches in smoothed analysis.

\paragraph{Our approach: preserving anti-concentration through
a  coupling-based reduction.}
We introduce a general technique for reducing smoothed analysis with
adaptive adversaries to the much simpler setting of oblivious
adversaries.  
{We consider adaptive adversaries that 
at each time step choose an input distribution with density function bounded
above pointwise by~$\tfrac{1}{\sigma}$ times that of the uniform.}
The crux of our approach is a coupling argument, 
namely a joint distribution that connects $T$ random variables $(X_1,
\dots, X_T)$ generated by an adaptive smooth adversary with $kT$ random
variables $Z_i^{(t)}$ for $i\in [k]$ and $t\in [T]$ that are generated
i.i.d.~from the uniform distribution.  A key aspect of this coupling
is a monotonicity property,  that for $k = \tilde{\Theta}(1/\sigma)$, with 
high probability,
$ \left\{ X_1 , \dots , X_T\right\} \subseteq \big\{ Z_i^{\left( j
  \right)} \mid i \in \left[ k \right] , j\in \left[ T \right]
\big\}$.

The properties of this coupling 
allow us to translate typical algorithms and proofs from the
setting of oblivious adversaries to that of adaptive adversaries.
For example, consider an algorithm that fails only when 
$X_1,\ldots,X_T$ ``concentrate,'' 
roughly meaning that many of the $X_i$'s land in an a priori chosen set
of small measure (this is a recurring theme in the
smoothed analysis of algorithms).
After substituting in $\big\{ Z_i^{\left( j \right)} \mid i \in [k] , j\in
[T] \big\} \supseteq \{X_1 , \dots , X_T\}$, the likelihood of this
event only increases.  (See ~\cref{sec:recipe} for precise statements.)
On the other hand, i.i.d.~uniform random
variables (the $Z_i^{(j)}$'s) have ideal anti-concentration properties
for a smoothed analysis.

The power of our coupling technique is in its versatility. To demonstrate this, we apply our coupling approach to
applications to online learning, online discrepancy minimization, and
dispersion in online optimization. In each of these problems, we
show that existing analyses for oblivious adversaries fundamentally
boil down to a suitable anti-concentration result.
For online learning --- where our work resolves an open problem of \cite{NIPS2011_4262} --- what matters is the anti-concentration of the
input instances in the symmetric difference between a hypothesis 
and its nearest neighbor in a
finite cover of the hypothesis class. 
For online discrepancy minimization, what matters is the
anti-concentration of correlations between discrepancy vectors and
inputs. For dispersion, what matters is the anti-concentration of
function discontinuities in small intervals.  After isolating these
key steps, we prove that the coupling approach can be used to lift
them (and the algorithmic guarantees that they lead to) to the general
case of adaptive adversaries.

%% file: formulation.tex
\subsection{Overview of our Results}
    
    \begin{table}
    \centering
    \small{

\begin{tabular}{|c|c|c|c|}
    \hline 
    & Worst Case & Stochastic/ Oblivious & Adaptive Smoothed  \\ 
    \hline
    Online Learning & $\tilde{\Theta} \left( \sqrt{T \cdot \ldim } \right) $  & $\tilde{\Theta} \left( \sqrt{T \cdot d } \right) $  & $\tilde{\Theta} \left( \sqrt{T \cdot d \log\left( 1 / \sigma \right) } \right)  $ \\ 
     & \cite{ben2009agnostic}  & \cite{haghtalab2018foundation}  & \cref{thm:regret-main} \\
    \hline 
    Online Discrepancy & $\Omega\left( \sqrt{T/ n} \right) $ & $O\left( \log\left( n T \right) \right)$\cite{ALS_Disc}  & $\tilde{O} \left( \log^2 \left( n T / \sigma \right) \right) $   \\
     & \cite{ten_lectures}  &  $O\left( \log^4\left( n T \right) \right)$  \cite{Bansal_Discrepancy}  & \cref{thm:main_discrepancy} (also isotropic) \\
    \hline 
    Dispersion & $\left( w , T \ell \right)$ & $\left( \sigma (T \ell)^{\alpha -1},  O\left( (T \ell)^{\alpha} \right)\right)$ & $\left( \sigma (T \ell)^{\alpha -1},  \tilde{O}\left( (T \ell)^{\alpha} \right)\right)$ \\ 
     &  $\forall w;$ (trivial) & \cite{Dispersion}  & \cref{thm:ada_disp}  \\
    \hline
\end{tabular}
}
\caption{\small{This table compares and summarizes the results of this paper and those from previous works. In this table, $T$ is the time horizon, $\sigma$ is the smoothness parameter, $d$ is the VC dimension of the hypothesis class in  online learning, $n$ is the dimension of the space for online discrepancy, $\ell$ is the number of discontinuities of 
piecewise Lipschitz functions in online optimization, and  $\alpha\in \left[ 0.5, 1 \right] $ is arbitrary. }}
\end{table}

Throughout this paper we consider $\sigma$-smooth adaptive adversaries. A $\sigma$-smooth distribution $\D$ is a distribution whose densities are bounded by $1/\sigma$ times the density of the uniform distribution over a domain. Formally this definition is 
captured as follows.

    \begin{definition}[$\sigma$-smoothness]
        Let $\X$ be a domain that supports a  uniform distribution $\U$.\footnote{Such as $\X$ that is finite or has finite Lebesgue measure.}
        A measure $\mu$ on $ \X $ is said to be $\sigma$-smooth if for all measurable subsets $A \subset \X$, we have $\mu \left( A \right) \leq \frac{\U\left( A \right)}{\sigma} $.
    \end{definition}
{
This parameterized definition of ``sufficiently concentrated'' is the standard one that has been used in smoothed analysis over the past decade, for example in all analyses of the smoothed running time of local search heuristics \cite{manthey_2021}.  It prevents an adversary from concentrating most of its probability mass near a specific worst-case input (as is necessary for any interesting results) without resorting to any parametric assumptions.
}

We focus on smoothed analysis of adaptive adversaries that at time $t$ pick a $\sigma$-smooth distribution $\D_t$ after having observed earlier instances $x_1\sim \D_1, \dots, x_{t-1}\sim \D_{t-1}$ and algorithmic choices. We denote an adaptive sequence of $\sigma$ distributions by $\adist$.
We use $\adist$ to model smoothed analysis of online learning, online discrepancy, and online optimization with an adaptive adversary.

\paragraph{Online Learning.}
We work with the setting of smoothed
  \emph{online adversarial and (full-information) learning}. In this setting, a learner and an adversary play a repeated game over $T$ time steps.
 For a labeled pair $s = (x, y)$ and a hypothesis $h\in \H$,
$\mathbb{I} \left[ h(x) \neq y \right] $ indicates whether $h$ makes a
 mistake on $s$.
In every time step $t\in [T]$ the learner picks a hypothesis $h_t$ and adversary picks a distribution $\D_t$ whose marginal on $\X$ is $\sigma$-smooth and then draws  $s_t\sim \D_t$.
 The learner then incurs penalty of   $\mathbb{I} \left[ h(x_t) \neq y_t \right] $.
 We consider an \emph{adaptive} $\sigma$-smooth adversary and denote it by $\adist$,
 where $\D_t$ is selected by an adversary that knows the algorithm and has observed $s_1, \dots, s_{t-1}$ and $h_1, \dots, h_{t-1}$.
Our goal is to design an online algorithm $\A$ such that expected regret against an adaptive adversary,
\[
\Ex[\regret(\A, \adist)]{:=}\Ex_{\adist}  \left[ \sum_{t=1}^T \mathbb{I} \left[ h_t(x_t) \neq y_t \right] 
 - \min_{h\in \H}  \sum_{t=1}^T  \mathbb{I} \left[ h(x_t) \neq y_t \right] 
\right]
\]
is sublinear in $T$. {This is the most well-studied domain for the application of our techniques.}

    In the worst case (without smoothness), 
\cite{ben2009agnostic} showed that the optimal regret in online learning is  characterized by finiteness of    a combinatorial quantity known as the Littlestone dimension, more formally, it is    $ \regret = \tilde{\Theta} \left(   \sqrt{  \lsd \left( \F \right) T  } \right)$.     
    Unfortunately, the Littlestone dimension can be large even for classes where the VC dimension is small. 
\cite{NIPS2011_4262, haghtalab2018foundation, haghtalab2020smoothed} considered the smoothed analysis of online learning and asked whether regret bounds that are characterized by finiteness of $\vc(\H)$ are possible. For the oblivious smooth adversaries, \cite{haghtalab2018foundation} answered this in the positive. For adaptive smooth adversaries however, the best-known bounds are $\tilde{\Theta} \left( \sqrt{T \cdot \log \mathcal{N}_{\left[ \, \right]}  } \right) $ where $\mathcal{N}_{\left[ \, \right]}$ denotes the \emph{bracketing number} which can be infinite even when $\vc(\H)$ is constant. 

    In this paper, we bridge the gap between smoothed analysis of online learning with adaptive and non-adaptive adversaries, answer an open problem of \cite{NIPS2011_4262,haghtalab2018foundation}, and show that regret bounds against an adaptive smooth adversary are nearly the same as those in agnostic offline learning. 

\medskip
\noindent\textbf{\cref{thm:regret-main}}(Informal)\textbf{.}\emph{
        Let $\H$ be a hypothesis class of VC dimension $d$. There is an algorithm $\A$ such that for any adaptive sequence of $\sigma$-smooth distributions $\adist$ achieves a regret of 
        \[
 \Ex[\regret(\A, \adist )] \in \tilde{O} \left(\sqrt{Td \ln \left( \frac{T}{d \sigma } \right) } +  d  \ln\left( \frac{T }{d \sigma } \right)   \right).
        \]
}

\medskip

{We complement this by a nearly matching lower bound as follows.}
\medskip

\noindent\textbf{\cref{thm:regret_lowerbound}} (Informal)\textbf{.}\emph{
        For every $d$ and $ \sigma $ such that $d \sigma\leq 1  $, there exists a hypothesis class  $ \H $ with VC dimension $d$ such that for any algorithm $\A$ there is a sequence of $\sigma$-smooth distributions $\D$ where     
        \begin{equation*}
       \Ex[\regret(\A, \D )] \in  \Omega\left(    \sqrt{T d \log \left( \frac{1}{\sigma d} \right) }  \right).
        \end{equation*}
}
\medskip

\paragraph{Online Discrepancy.}
Our starting point is the Koml\'os problem.
In this online discrepancy problem, we are given an online sequence of vectors $v_1, \dots, v_{T}$ with $\norm{v_i}_2 \leq 1$. Upon seeing $v_i$ we need to immediately and irrevocably assign  sign $\epsilon_i \in \{-1, +1\}$ to $v_i$. Our goal is to keep the following discrepency vector small
    \begin{equation*}
       \max_{t\in [T]} ~~\norm{ \sum_{i=1}^t \epsilon_i v_i }_{\infty}.
    \end{equation*}
    This problem is interesting for various norms on the inputs and the discrepancy, here we  restrict ourselves to $\ell_2$ and $\ell_{\infty}$ norms, respectively.
    
    It is not hard to see that in the fully adaptive setting, the adversary can pick a vector orthogonal to the current discrepancy vector leading to the $\ell_{\infty}$ discrepancy  norm of growing as $ O \left( \sqrt{T/n} \right)$.
    To overcome this, stochastic versions of this problem have been considered where vectors $v_i$ are picked from a \emph{fixed  and known} distribution from set of vectors with $\norm{v_i} \leq 1$.
   \cite{Bansal_Discrepancy} uses a potential-based approach to obtain a bound of $ O\left( \log^4 \left( n T \right) \right) $ for the stochastic setting. 
    \cite{ALS_Disc} strengthens these results to hold for any sequence of inputs that is chosen by an oblivious (even deterministic) adversary and obtains $O \left( \log \left( n T \right) \right)$ on the discrepancy.
    
We consider adversaries that pick a $\sigma$-smooth distribution $\D_t$ at time $t$ after having observed the earlier instances $v_1, \dots, v_{t-1}$ and their assigned signs $\epsilon_1, \dots, \epsilon_{t-1}$ and then draw $v_t\sim \D_t$. 
We bound the discrepancy of this setting by $O\big( \log^{2} \left( n T \right) \big) $.

\medskip
\noindent\textbf{\cref{thm:main_discrepancy}} (Informal)\textbf{.}\emph{
     Let $v_1, \dots, v_T$ be chosen from an adaptive sequence of $\sigma$-smooth and isotropic distributions $\adist$.%
        Then, there is an online algorithm for deciding the sign $\epsilon_i$ of $v_i$, such that with high probability 
        \begin{equation*}
          \max_{t\leq T}~  \norm{ \sum_{i = 1}^t \epsilon_i v_i   }_{\infty} \leq O\left( \log^2 \left( \frac{ T n  }{ \sigma} \right) \right). 
        \end{equation*}
}

We note that our adaptive isotropic assumption is {mild}, as even for the case of stochastic uniform inputs (which are isotropic) the first $\mathrm{polylog}(nT)$ bound was introduced by \cite{bansal2020online} in STOC 2020. 
{Our next theorem further justifies the use of  isotropic distributions  by showing that smoothness alone is not enough to achieve a $\mathrm{polylog}(nT/\sigma)$ bound on discrepancy in presence of adaptive adversaries.}

\medskip
\noindent\textbf{\cref{thm:disc_lowerbound}} (Informal)\textbf{.}\emph{
    For any online algorithm, there is an adaptive sequence of $ \left(  \frac{1}{20 n^2T^2} \right) $-smooth distributions on the unit ball such that,  we have 
    \begin{equation*}
        \norm{ \sum_{i=1}^T  \epsilon_i v_i }_{\infty} \geq \Omega \left( \sqrt{\frac{T}{n} } \right)
    \end{equation*}
    with probability $1 - \exp\left( - \frac{T}{12} \right)$.
}    

\paragraph{Dispersion in Online Optimization.}
In the online optimization setting, an adversary chooses a sequence of loss functions $u_1, \dots, u_T$ and at each time step the learner picks an instance $x_t$ in order to minimize regret 
\[
 \sum_{t=1}^T u_t(x_t)   - \min_x \sum_{t=1}^T u_t(x).
\]
\cite{Dispersion} studied this problem for piecewise Lipschitz functions and showed that regret is characterized by a quantity called \emph{dispersion}.  At a high level, a sequence of functions is called \emph{dispersed} if no ball of small width intersects with discontinuities of many of these functions.

 \begin{definition}[Dispersion, \cite{Dispersion}]
   Let $u_1 , \dots , u_T : [0,1] \to \br$ be a collection of functions such that $u_i$ is piecewise Lipschitz over a partition $\mathcal{P}_i$ of $[0,1]$. 
   We say that a partition $\mathcal{P}_i$ splits a set $A$ if $A$ intersects with at least two sets in $\mathcal{P}_i$.  
   The collection of functions is called \emph{$\left( w , k \right)$-dispersed} if every interval of width $w$ is split by at most $k$ of the partitions $\mathcal{P}_1,\dots,\mathcal{P}_T $.     This definition naturally extends to loss functions over $\br^d$ as well.
    \end{definition}

Additionally, \cite{Dispersion} showed that when an oblivious $\sigma$-smooth adversary picks  the discontinuities of piecewise Lipschitz functions, the resulting sequence is with high probability $\left( \sigma (T \ell)^{\alpha -1},  O\left( (T \ell)^{\alpha} \right)\right)$-dispersed, where $\alpha$ can be any value in $[0.5,1]$ where $\ell$ is the number of discontinuities. We extend this result to the case of adaptive smooth adversaries and recover almost matching bounds on dispersion parameters.
Our work shows that adaptive smooth adversaries generate dispersed sequences in online optimization. This allows us to extend the power of algorithms designed for dispersed sequences, such as efficient online and private batch optimization~\cite{Dispersion}, to the larger setting of adaptive adversaries.

\medskip
\noindent\textbf{\cref{thm:ada_disp}} (Informal)\textbf{.}\emph{
    Let $u_1 \dots u_T$ be functions from $\left[ 0,1 \right] \to \br $ that are piecewise Lipschitz with $\ell$ discontinuities each picked by a $\sigma$-smooth adaptive adversary. 
        Then, for any $\alpha \geq 0.5 $, the sequence of functions $u_1 \dots u_T $ is $( \sigma(T \ell)^{\alpha -1}, \tilde{O}\left( \left( T \ell  \right)^{\alpha }  \right) )$-dispersed.  
}
\medskip

%% file: intro_technical_overview.tex
\section{Overview of the Techniques and Analysis}

We introduce a general technique for reducing smoothed analysis with
adaptive adversaries to the much simpler setting of oblivious
adversaries. 
Our main general technique is a \emph{coupling} argument between random variables that are generated by an adaptive smooth adversary and those that are generated i.i.d.~from a uniform distribution. 
This coupling, that is a joint distribution between two random processes,
demonstrates structural properties that are ideal for preserving and analyzing anti-concentration properties of smooth adversaries. 
This allows us to tap into existing techniques and algorithms that are designed for oblivious smooth adversaries and only rely on some anti-concentration properties of the input.

 We first give an overview of our coupling technique and its analysis in \cref{sec:coupling_overview} and then in \cref{sec:recipe} we give a general framework for applying coupling for smoothed analysis with adaptive adversaries. 

\subsection{Coupling Definition and Theorem statement} \label{sec:coupling_overview}
In this section, we will give an overview of the coupling between smooth adaptive adversaries and the uniform distribution. 
A \emph{coupling} is a joint distribution between two 
random variables, or random processes, such that the marginals of this coupling are distributed according to the specified random variables.   
A more formal definition of a coupling is as follows.

    \begin{definition}[Coupling]
        Let $\mu$ and $\nu$ be two probability measures on the probability space $ \left( \X , \mathscr{F} \right) $ respectively. 
        Then, a coupling between $\mu$ and $\nu$ is a measure $ \gamma  $  on $ \left( \X \times \X , \mathscr{F} \otimes \mathscr{F}  \right) $ such that for all $A \in \mathscr{F}  $, we have %
            $\gamma \left( A \times X \right) = \mu\left( A \right) \text{ and }  \gamma \left(  X \times A \right) = \nu\left( A \right).   $
        This definition can be generalized in a natural way to multiple measures. 
    \end{definition}

Our main coupling theorem states that given any adaptive sequence of $\sigma$-smooth distributions, $\adist$, there is a coupling between a random sequence $(X_1, \dots,X_T)\sim \adist$ and uniformly distributed random variables $Z_{i}^{\left( t \right)} $ such that 
(with high probability) the set of uniform random variables includes set of adaptively generated $\sigma$-smooth variables.

\begin{theorem} \label{thm:main_coupling_overview}
    Let $ \adist $ be an adaptive sequence of $\sigma$-smooth distribution on $\X$. 
    Then, for each $ k >0  $,  there is a coupling $\Pi$ such that $ \left( X_1 ,  Z_{1}^{ \left( 1 \right)} , \dots , Z_{k}^{\left( 1 \right)}   , \dots, X_t , Z_{1}^{ \left( t \right)} , \dots , Z_{k}^{\left( t \right)}  \right) \sim \Pi $ satisfy
    \begin{itemize}
        \item[a.] $X_1 , \dots , X_t$ is distributed according $\adist$. 
        \item[b.] $ Z_i^{\left( j \right)} $ are uniformly and independently distributed on $\X$. 
        \item[c.] $ \left\{ Z_i^{\left( j \right)} \mid j \geq t, i\in[k]\right\} $ are uniformly and independently distributed on $\X$, conditioned on $X_1 , \dots, X_{t-1}$.  
        \item[d.] With probability at least $1 - t \left( 1 - \sigma \right)^{k}  $, $ \left\{ X_1 , \dots , X_t\right\} \subseteq \left\{ Z_i^{\left( j \right)} \mid i \in \left[ k \right] , j\in \left[ t \right] \right\} $ . 
    \end{itemize}
\end{theorem}

    The key aspect of this theorem is the monotonicity property 
     $ \left\{ X_1 , \dots , X_t\right\} \subseteq \big\{ Z_i^{\left( j \right)} \mid i \in \left[ k \right] , j\in \left[ t \right] \big\}$ that holds with high probability. 
     This monotonicity and the fact that $Z_i^{\left( t \right)}$ are uniform are the crucial properties that allow us to reduce algorithms design and analysis against online adaptive adversaries to those designed against oblivious stochastic adversaries. 
   We will give examples of how this coupling will be used in \cref{sec:recipe}. 

In the remainder of this section, we give an overview of the construction of this coupling and the proof sketch for \cref{thm:main_coupling_overview}. 
For ease of exposition, we mainly restrict ourselves to the finite universe $ \X =  \left[ n \right]$ and  work with smooth distributions that are uniform on an adaptively chosen subsets of size at least $\sigma n$ of the universe.  We prove the theorem in its full generality in \cref{sec:main_coupling}.

Let us first consider a single round of coupling between a random variable that is uniformly distributed over $S\subseteq [n]$ of size $\sigma n$, and the uniform random variables over $[n]$.
Draw $k$ samples $Y_1,\dots,Y_k $ from the uniform distribution on $\left[ n \right]$.
If $Y_i \notin S $, then let  $Z_i = Y_i$.
Otherwise (that is when $Y_i \in S$) draw a fresh $ \tilde{W}_i$ uniformly from $S$ and let  $Z_i = \tilde{W}_i$.
We next define $X_1$. If for all $i\in[k]$ we have $Z_i \notin S $, then let $X_1$ be a uniform pick from the set $S$, otherwise 
let $X_1$ be uniformly chosen from the set of all $\tilde{W}_i$s.

 It is clear that $X_1 $ is uniformly distributed on $S$ since it is either equal to a $\tilde{W}_i$, which is itself uniformly distributed over $S$, or is directly drawn uniformly from $S$.
    It is not hard to see that $Z_i$s are independent, because they are functions of $Y_i$s and $\tilde{W_i}  $ which are all mutually independent. 
    Furthermore, for any $\ell \notin S$, we have 
    $ \Pr\left[ Z_i = \ell  \right] = \Pr\left[ Y_i = \ell \right] = 1/n $. Similarly, for $ \ell \in S  $, 
    \[
     \Pr\left[ Z_i = \ell \right] = \Pr\left[ Y_i \in S \right] \Pr\left[ \tilde{W}_i = \ell \right] = \sigma \times \frac{1}{\sigma n} = \frac 1n.\]
    This shows that $Z_i$ are uniformly and independently distributed. 
    As for monotonicity, note that $X_1 \notin \left\{ Z_1 \dots Z_k \right\} $ only if no $Z_i$ was in $S$, which occurs only with probability $ (1- \sigma)^{k}$. 

    Next we create a coupling for adaptive $\sigma$-smooth distributions $\adist$. 
    Recall that in this setting an adaptive sequence corresponds to $X_{\tau}  $ being sampled uniformly from a set $S_\tau \left( X_1, \dots, X_{ \tau-1} \right)$, i.e., the set at time $\tau$ is adaptively chosen given the earlier realizations.
    We construct the coupling inductively using the same ideas discussed for the single round coupling, but at each step using $ S_{\tau} \left( X_1, \dots, X_{ \tau-1} \right)  $.  
    Formally, the coupling is as below:

    \begin{framed}
        \begin{itemize}
            \item For $j = 1 \dots t$, 
            \begin{itemize}
                \item Draw $ k = \alpha \sigma^{-1} $ samples $Y^{ \left( j \right) }_1 , \dots , Y^{(j)}_{k}$ from the uniform distribution. 
                \item If  $Y^{\left( j \right)}_i \notin S_j \left( X_1 , \dots , X_{j-1} \right) $, then $Z^{\left( j \right)}_i = Y^{\left( j \right)}_i$. 
                \item Else, for $i$ such that $Y^{\left( j \right)}_i \in S_j \left( X_1 , \dots , X_{j-1} \right) $, sample $\tilde{W}^{\left( j \right)}_i$ {uniformly and independently from} 
                $S_j \left(  X_1 , \dots , X_{j-1} \right)  $ and set $Z^{\left( j \right)}_i =\tilde{W}_i^{\left( j \right)}$. 
                \item 
If for all $i$, $Y^{\left( j \right)}_i \notin S_j \left( X_1 , \dots ,X_{j-1} \right) $, then sample $X_j$ uniformly from $S_j \left( X_1 , \dots ,X_{j-1} \right)   $.
                Otherwise, pick $X_j  $ uniformly from $ \left\{ \tilde{W}_i^{\left( j \right)} \mid i\in[k]   \right\} $.
            \end{itemize}
            \item Output $ \left( X_1 ,  Z_{1}^{ \left( 1 \right)} , \dots , Z_{k}^{\left( 1 \right)}   , \dots, X_t , Z_{1}^{ \left( t \right)} , \dots , Z_{k}^{\left( t \right)}  \right)  $. 
        \end{itemize}
    \end{framed}
    
    We prove that this coupling works inductively. 
    Fixing $X_1, \dots, X_{\tau-1} $, we get $S_{\tau}\left( X_1, \dots, X_{ \tau-1}\right)$. 
    Note that the coupling in stage $ \tau$ is similar to the single round coupling. 
    From a similar argument, we get that $X_{\tau}$ is distributed uniformly on $S_{\tau} \left(X_1 ,  \dots , X_{\tau-1} \right)$. 
    Similarly, one can argue that $Z_1^{\left( \tau \right)}, \dots, Z_{k}^{\left( \tau \right)} $ are independent and uniform. 
    The monotonicity property follows from the monotonicity in each stage and a union bound. 

    The only other main property that needs to be argued is that $Z_1^{\left( \tau \right)}, \dots, Z_{k}^{\left( \tau \right)}$ are independent of all the past random variables $X_1, \dots, X_{\tau-1} $ and $ \left\{ Z_{i}^{\left( j \right)} \mid i \in \left[ k \right], j \leq \tau-1 \right\}   $. 
    The key property needed here is that in the single-round coupling, the distribution of $Z_i$ is oblivious to the choice of the set $S$.
    We prove this formally in \cref{sec:main_coupling}.
     Informally, this can also be seen by noting that the one step coupling  above is equivalent to the coupling where $Z_j$ are all sampled independently and uniformly and $X_1$ is set to a random $Z_j$ that is in the set $S$, or when none of them are in this set, it is  sampled independently. %
    This in particular ensures that 
$\left\{ Z_i^{\left( j \right)} \mid j \geq t, i\in[k]\right\}$ are uniform and independent of the past.

    Note that the above reasoning works as long as the sets $S_j(X_1, \dots, X_{j-1})$ have at least $n\sigma$ elements.     
    In order to move from the special case of uniform distributions on $S_j(X_1, \dots, X_{j-1})$s, we note that smooth distributions are convex combinations of uniform distributions on subsets of size $\geq \sigma n $. 
    \begin{lemma}
        Let $\mathcal{P}$ be the set of $\sigma$-smooth distributions on $\left[ n \right]$ and let $\mathcal{P}_0$ be the set of distributions that are uniform on subsets of size at least $\sigma n $. 
        Then, 
$                \mathcal{P} = \mathrm{conv} \left( \mathcal{P}_0 \right).$ 
    \end{lemma}

In particular, this implies that for each $ \sigma  $-smooth distribution $\mathcal{D}$, there is a distribution $\mathcal{S}_{\mathcal{D}} $ on subsets of size at least $\sigma n$ such that sampling from $\mathcal{D} $ can be achieved by first sampling $S \sim \mathcal{S}_{\mathcal{D}} $ and then sampling uniformly from $S$. 
 
For infinite domains, similar argument can be made using the Choquet integral representation theorem which gives a way to represent smooth distributions as convex combinations of uniform distributions on sets of large measure. 
Putting this together leads to \cref{thm:main_coupling_overview}.

\input{recipe}

%% file: recipe.tex
\subsection{The General Framework for applying the Coupling.}
\label{sec:recipe}

In most applications where smoothed analysis has led to significant improvements over the worst-case analysis,
these improvements hinge on the proof techniques and algorithmic 
approaches that leverage 
the anti-concentration properties of the smoothed input.
However, as the process of creating an input becomes more and more adaptive, that is, as the adversary correlates the distribution of the current input with the realizations of earlier inputs and decisions 
the randomness and anti-concentration properties of the input and the state of the algorithm may weaken.
Additionally, correlations between future and past instances present novel challenges to the methodology used against oblivious smooth adversaries, which often rely heavily on the independence of the input.
Our coupling approach overcomes these challenges in two ways. First, by coupling an adaptive smooth process with a non-adaptive uniform process, it implicitly shows that anti-concentration properties of the input and the algorithm do not weaken significantly in presence of adaptive adversaries. 
Second, it allow us to lift algorithmic ideas and proof techniques that  have been  designed for oblivious smooth or stochastic adversaries to design and analyze algorithms that have to interact with adaptive smooth adversaries.

An important property of our coupling is its monotonicity, i.e.,
with high probability, $\left\{ X_1 , \dots , X_t\right\} \subseteq \big\{ Z_i^{\left( j \right)} \mid i \in [k] , j\in [t] \big\}$.
This monotonicity property paired with the fact that $Z_i^{(t)}$ are i.i.d~uniform variables are 
especially useful for lifting algorithms and proof techniques from the oblivious world that rely on anti-concentration. 
That is, if an algorithm's failure mode is only triggered when $X_1 , \dots , X_t$ concentrate,
then replacing in $\big\{ Z_i^{\left( j \right)} \mid i \in [k] , j\in [t] \big\} \supseteq  \{X_1 , \dots , X_t\}$ can only increase the likelihood of hitting the failure mode.
On the other hand, i.i.d.~uniform random variables $Z_i^{(t)}$s demonstrate excellent  anti-concentration properties that are superior to most other offline stochastic or oblivious smooth distributions.
This shows that existing techniques and algorithms that work well in the stochastic or oblivious smooth settings will continue to work well for adaptive smooth adversaries.

As a general blueprint for using our coupling for smoothed analysis with adaptive adversaries, 
first consider how you would handle smooth oblivious or stochastic adversaries and identify 
steps that rely on an anti-concentration property. Sometimes, this is more easily done by identifying where existing approaches rely on the obliviousness and stochasticity of the  adversaries and then finding concentration properties, potential functions, or other monotone set functions that implicitly measure concentration of some measure.
Next, apply the coupling to replace $T$ adaptive smooth random variables with $Tk$ i.i.d~uniform random variables and show that the previous anti-concentration (or other monotone properties) are only moderately affected by the fact that we have a larger number of random variables.
Finally, use the original algorithm or technique for leveraging anti-concentration  and complete the proof.

In the remainder of this section, we show how the above blueprint can be applied to three important examples from online learning, discrepancy, and optimization.

\paragraph{Online Learning.}
One key property that enables learnability in the offline agnostic, offline PAC, and oblivious smooth online setting is that a hypothesis class $\H$ can be approximated via a finite cover $\H'$ and algorithms such as ERM and Hedge can be run on $\H'$ without incurring a large error~\cite{haghtalab2018foundation,haghtalab2020smoothed}. 
This is due to the fact that the performance of the best hypothesis in $\H$ is closely approximated by the performance of the best hypothesis in $\H'$ when instances are drawn from an offline stochastic or an oblivious sequence of smooth distributions.
At the heart of this property is an anti-concentration of measure in the class of symmetric differences between hypotheses $h\in \H$ and their proxies $h'\in \H'$. More formally, 
for a fixed distribution $\D$, such as the uniform distribution, consider $\H'\subseteq \H$ that is an $\epsilon$-cover of $\H$ with respect to $\D$ so that for every hypothesis $h\in \H$  there is a proxy $h'_h\in \H'$ with $\Pr_\D[h(x)\neq h'_h(x)]\leq \epsilon$. %
The set $\H'$ is a good approximation for $\H$ under distribution $\D$ if not too many instances fall in any symmetric difference, that is, if with high probability, 
\[ \forall h\in \H, \frac{1}{T} \sum_{t=1}^T \mathbb{I} \left[ h(x_t) \neq h'_{h}(x_t) \right] \lesssim \epsilon.
\]
In the offline or oblivious smooth online setting this is done by leveraging the independence between $x_t$s and using techniques from the VC theory to show that each function $h\Delta h_{h'}$ is close to its expectation.

We note that $\max_{h\in \H} \sum_{x\in S} \mathbb{I} \left[   h(x) \neq h'_{h}(x) \right]$, which measures concentration, is a monotone set function that only increases when 
replacing random variables $X_1, \dots, X_T$ with random variables $\{Z_i^{(t)} \mid i\in [k], t\in [T]\} \supseteq \{X_1, \dots, X_T\}$.
This shows that the concentration of measure over a $T$-step adaptive smooth sequence of distributions $\adist$ is bounded by the concentration of measure over a $kT$ draws from the  uniform distribution. We can now use the anti-concentration properties of i.i.d.~uniform 
random variables and techniques from the VC theory (which were used for the oblivious smooth and stochastic case) to show that each function $h\Delta h_{h'}$ is close to its expectation.

\paragraph{Online Discrepancy.}

Most existing approaches for designing low discrepancy algorithms, such as~\cite{Bansal_Discrepancy,bansal2020online} control and leverage anti-concentration properties of the discrepancy vector and its correlations. %
In particular, 
\cite{Bansal_Discrepancy} introduces a potential function $\Phi_t$ that, roughly speaking, is $\exp(\lambda d_t^\top W)$ where $W$ is a mixture of the future random variables and  test directions. 
They use the fact that $X_t$s are generated i.i.d~from a fixed and known distribution to bound the tail probabilities for $\exp(\lambda d_{t-1}^\top X_t) > \Phi_{t-1}$.

Note that the event $\exp(\lambda d_{t-1}^\top X_t) > \Phi_{t-1}$ is  monotone, i.e., 
\[\sum_{i \in[k]} \exp(\lambda d_{t-1}^\top Z_i^{(t)}) \geq \exp(\lambda d_{t-1}^\top X_t), \]
when $X_t \in \{Z_i^{(t)} \mid i\in [k]\}$. Therefore, the coupling argument allows us to bound the tail probability of crossing the threshold $k\Phi_{t-1}$.
In other words, we bound the tail probabilities of having large correlation with an adaptive $\sigma$-smooth variable $X_t$ in terms of the tail probability of having correlations with at least one of $k$ i.i.d.~uniform random variables.  

With these tail bounds in place, we  now have a high probability event that $\exp\big( \lambda d_{t-1}^\top X_t \big) \leq k\Phi_{t-1}$.
Then, as \cite{Bansal_Discrepancy} argues, when $\Phi_{t-1}$ is large and as result $\lambda d_{t-1}^\top X_t$ by comparison cannot be large, there will be only a small increase in the potential function. Since  $\Phi_t$s also measure correlations with the test vectors, an upper bound on $\Phi_t$s also bounds the discrepancy. 

It is important to note that discrepancy itself is not a monotone set function as additional vectors can significantly reduce the discrepancy and stop it from growing it large over time. However, anti-concentration techniques that are at the core of analyzing discrepancy are monotone and therefore can be easily used with our coupling.

\paragraph{Dispersion.}
At its core,  dispersion is an anti-concentration property for the number of function discontinuities that  fall in any sufficiently small interval.
 Existing results of \cite{Dispersion} leverages anti-concentration of oblivious smooth adversaries, who generate independently distributed discontinuities, and argues that the resulting sequence is dispersed with high probability. That is, when the $j$th discontinuity of the $t$th function, $d_{t,j}$, is drawn independently, with high probability for all intervals $J$ with small width,  $\sum_{t,j} \mathbb{I} \left[   d_{t,j}\in J \right]$ is small. 
\cite{Dispersion} proves this using the independence between $d_{t,j}$s and the fact that VC dimension of the class of intervals is a constant.

In an approach that mirrors our online learning analysis, we emphasize that 
\[\max_J \sum_{d_{t,j}\in S}^T \mathbb{I} \left[  d_{t,j}\in J   \right] \]
that measures concentration of function discontinuities is  a monotone set function over $S$ and only increases when replacing random variables $d_{i,t}$s with random variables $\{Z_i^{(t,j)} \mid i\in [k], t\in [T], j\in[\ell] \} \supseteq \{d_{t,j} \mid j\in [\ell], t\in[T] \}$.  
This shows that the concentration of discontinuities over a $T\ell$-step adaptive smooth sequence of distributions $\adist$ is bounded by the concentration of discontinuities from a $kT\ell$-step uniform distribution. We can now use the anti-concentration properties of uniform and independent random variables and the fact that the VC dimension of intervals is small to show that adaptive smooth adversaries also create dispersed sequences.

%% file: regret.tex
\section{Regret Bounds against Smooth Adaptive Adversary} \label{sec:RegretBounds}
In this section, we  obtain regret bounds against adaptive smooth adversaries that are solely defined in terms of VC dimension of the hypothesis class and the smoothness parameter.

Recall that an adaptive adversary at every time step  $t\in [T]$ chooses $\D_t$ based on the actions of the
  learner $h_1,\dots, h_{t-1}$ and the realizations of the previous
  instances $(x_1, y_1), \dots,  (x_{t-1}, y_{t-1})$ and then samples $(x_t, y_t) \sim \D_t$. 
  Our main result in this section is as follows.
  
\begin{theorem} [Regret upper bound]
\label{thm:regret-main}
        Let $\H$ be a hypothesis class of VC dimension $d$. There is an algorithm $\A$ such that for any adaptive sequence of $\sigma$-smooth distributions $\adist$ achieves a regret of 
        \[
 \Ex[\regret(\A, \adist )] \leq \tilde{O} \left(\sqrt{Td \ln \left( \frac{T}{d \sigma } \right) } +  d  \ln\left( \frac{T }{d \sigma } \right)   \right).
        \]
        In the above $ \tilde{O} $ hides factors that are  $\mathrm{loglog}\left( \nicefrac{T}{d \sigma} \right) $.
            \end{theorem}

        We complement this result by providing nearly matching lower bounds.
        We show that \cref{thm:regret-main} is tight up to a {multiplicative $ \mathrm{polylog}(T)$ and $\mathrm{polyloglog}(1/\sigma d)$ factors and an additive $  d  \log\left(   \nicefrac{T}{d \sigma }  \right)  $ term.}
        We provide a proof of \cref{thm:regret_lowerbound} in \cref{sec:regretlowerproof}. 

    \begin{theorem}[Regret lower bound] \label{thm:regret_lowerbound} 
        For every $d$ and $ \sigma $ such that $d \sigma\leq 1  $, there exists a hypothesis class  $ \H $ with VC dimension $d$ such that for any algorithm $\A$ there is a sequence of $\sigma$-smooth distributions $\D$ where     
               \begin{equation*}
              \Ex[\regret(\A, \D )] \in  \Omega\left(    \sqrt{dT \log \left( \frac{1}{\sigma d} \right) }  \right).
               \end{equation*}
    \end{theorem}

In order to prove \cref{thm:regret-main}, we follow the general approach for using our coupling theorem (\cref{thm:main_coupling_1}) as outlined in \cref{sec:recipe}. That is, in \cref{sec:regret-overview}, we first review the algorithmic result of \cite{haghtalab2018foundation} for obtaining regret bounds against \emph{non-adaptive} smooth adversaries and identify steps for which non-adaptivity is crucial for that approach. In \cref{sec:obli-adap}, we then alter those steps to work for adaptive smooth adversaries via the coupling argument. Lastly, in \cref{sec:regretmain}, we combine the steps to complete the proof of \cref{thm:regret-main}.

\subsection{Overview of Existing Approaches and their Need for Obliviousness}
\label{sec:regret-overview}
\cite{haghtalab2020smoothed,haghtalab2018foundation}
considered regret-minimization problem against non-adaptive smooth adversaries.
This approach considered an algorithm $\A$ that uses Hedge or any other standard no-regret algorithm on a finite set $\H'$. $\H'$ is chosen to be an $\epsilon$-cover of $\H$ with respect to the uniform distribution. 
It is not hard to see (e.g., \cite[Equation (1)]{haghtalab2020smoothed}) that regret of algorithm $\A$ decomposes to the regret of Hedge on the cover $\H'$ and the error caused by approximating $\H$ by its cover $\H'$ as follows.
\begin{equation}\label{eq:Regret_Decomp}
\Ex[\regret(\A, \adist )] \leq  O\left( \sqrt{T \ln(|\H'|)} \right)  + \mathbb{E}_{\adist}\left[\max_{h\in \H}\min_{h'\in \H'} \sum_{t=1}^T 1\left(h(x_t) \neq h'(x_t) \right)  \right]
\end{equation}
Given that any hypothesis class $\H$ has an $\epsilon$-cover of size $(41/\epsilon)^{\vcd(\H)}$ (see \cite{HAUSSLER1995217} or \cite[Lemma 13.6]{boucheron2013concentration}) the first term of \autoref{eq:Regret_Decomp} can be directly bounded by $O\left(\sqrt{T\ \vc(\H) \ln(1/\epsilon)} \right)$. To bound the second term of \autoref{eq:Regret_Decomp}, 
for any $h\in \H$ consider the $h' \in \H'$ that is the proxy for $h$, i.e., $g_{h, h'} = h\Delta h'$ is such that $\Ex_{x\sim U} [g_{h,h'}(x)]\leq \epsilon$, where $U$ is the uniform distribution over $\X$. Let $ \G = \{g_{h, h'}\mid \forall h\in \H \text{ and the corresponding proxy } h'\in \H' \}$. Note that, 
\begin{align}
\Ex_{\adist}\left[\sup_{h\in \H}\inf_{h'\in \H'} \sum_{t=1}^T 1\left(h(x_t) \neq h'(x_t) \right)  \right] \leq 
\Ex_{\adist}\left[ \sup_{g \in \G} \sum_{t = 1}^T g\left( x_t \right) \right].
\end{align}
Note that for any fixed $g_{h,h'}\in \G$ and even an adaptive sequence of $\sigma$-smooth distributions, 
$\Ex_{\adist}[\sum_{t=1}^T g_{h,h'}(x_t)] \leq \sigma^{-1} \Ex_{\U}[\sum_{t=1}^T g_{h,h'}(x_t)] \leq T\epsilon/\sigma$.

Up to this point, the above approach applies equally to adaptive and non-adaptive adversaries. 
It remains to establish that with small probability over all (infinitely many) functions in $\G$, the realized value of $g$ is close to its expected value. This is where existing approaches rely on obliviousness of the adversary. \emph{When the adversary is non-adaptive, instances $x_t\sim \D_t$ are independently (but not necessarily identically) distributed.} Existing approaches such as \cite{haghtalab2018foundation} leverage the independence between the instances to  use the double sampling and symmetrization tricks from VC theory and 
establish a uniform convergence property even when instances are not identically distributed. That is, when $\nadist$ is a \emph{non-adaptive} sequence of smooth distributions,
\begin{equation}
\Ex_{\nadist}\left[ \sup_{g \in \G} \sum_{t = 1}^T g\left( x_t \right) \right] \leq \frac{T\epsilon}{\sigma} + O\left( \sqrt{Td \ln  \left( \frac{T}{\sigma} \right)} \right)
\label{eq:nadist-bound}
\end{equation}
Using $\epsilon = \sigma T^{-1/2}$ in \autoref{eq:nadist-bound} and \autoref{eq:Regret_Decomp} gives an upper bound on the regret against an oblivious smooth adversary that only depends on VC dimension of $\H$ and the smoothness parameters.

\subsection{Reducing Adaptivity to Obliviousness via the Coupling} \label{sec:obli-adap}
We emphasize that \autoref{eq:nadist-bound} is the only step in existing approach that relies on the obliviousness of the adversary. In this section, we show how the coupling lemma can be used to obtain an upper bound analogous to the \autoref{eq:nadist-bound} for adaptive adversaries.  The main result of this section is as follows,

\begin{lemma} \label{lem:net_bound}
Let $\G$ be defined as described in \cref{sec:regret-overview}, $d = \vc(\H)$, and let $\adist$ be an adaptive sequence of $\sigma$-smooth distributions. We have
            \begin{equation*}
                \Ex_{\adist} \left[ \sup_{g \in \G  } \sum_{i = 1}^T g\left( x_i \right) \right] \leq O\left( \sqrt{ \frac{\epsilon}{\sigma} T \ln(T)\ d \ln\left( 1 / \epsilon \right) } + T \ln(T) \frac{\epsilon }{\sigma} \right) 
            \end{equation*}
            for any $\epsilon > \frac{\sigma d \log \left( 4 e^2 / \epsilon  \right)}{5 T \ln(T)}   $. 
\end{lemma}

\begin{proof}[Proof of \cref{lem:net_bound}]
Here we bound the value of a  $T$-step adaptive process. To prove this lemma, we use the coupling described in \cref{sec:coupling_overview} to reduce the problem of bounding the value of a $T$-step adaptive process  by the value of the a $\tilde O(T/\sigma)$-step uniform process. We then  bound the value of the uniform process using the fact that  uniform process is an oblivious process.
\begin{claim} \label{claim:couple-regret} Let  $\alpha = 10 \ln(T)$ and $k = \alpha / \sigma$, and let $\U$ denote the uniform distribution over the domain. We have
\[ \Ex_{\adist} \left[ \sup_{g \in \G }\ \sum_{i = 1}^T g\left( x_i \right)  \right] \leq 
 T^2 \left( 1 - \sigma \right)^{\frac{\alpha}{\sigma}}  + \Ex_{\U}\, \left[ \sup_{g \in \G } \sum_{\substack{i \in[k]\\ j\in[T]}} g\left( Z^{\left( j \right)}_i \right)   \right]. 
 \]
\end{claim}
\begin{proof}[Proof of \cref{claim:couple-regret}]
        Consider the coupling $X_1 , \dots X_T , Z_1^{\left( 1 \right)} , \dots Z_k^{\left( T \right)} $ described in 
        \cref{sec:Coupling} for 
for $k = \alpha / \sigma$ and $\alpha = 10 \ln(T)$.
        We will denote this by $\Pi$.  
        First note that every $g\in \G$ is positive, since it is a  symmetric difference between two functions $h$ and $h'$. Therefore, for any two sets $A$ and $B$, such that $A \subseteq B$, we have
                \begin{equation*}
            \sup_{g \in \G  }  \sum_{x \in A } g(x) \leq \sup_{g \in \G  }  \sum_{x \in B } g(x)     
        \end{equation*}
        Let $\mathcal{E}$ denote the event $ \left\{ X_1 , \dots , X_T\right\} \nsubseteq \left\{ Z_i^{\left( j \right)} \mid { i \in \left[ k \right] , j\in \left[ T \right] }  \right\}$. 
        From \cref{thm:main_coupling}, we know that $\Pr\left[ \mathcal{E} \right] \leq T \left( 1 - \sigma \right)^{\frac{\alpha}{\sigma}}$. 
        Moreover, from \cref{thm:main_coupling} we have that $X_1 \dots X_T$ is distributed according to $\adist$ and $Z_i^{\left( j \right)}$ are i.i.d according to $\U$, thus
        \begin{equation}
            \Ex_{\adist} \left[ \sup_{g \in \G } \sum_{i = 1}^T g\left( x_i \right)  \right] = \Ex_{\Pi} \left[ \sup_{g \in \G } \sum_{i = 1}^T g\left( X_i \right)  \right]  \text{ and }
            \Ex_{\U}\left[ \sup_{g \in \G } \sum_{\substack{i \in[k]\\ j\in[T]}} g\left( Z^{\left( j \right)}_i \right)   \right] =             \Ex_{\Pi}\left[ \sup_{g \in \G } \sum_{\substack{i \in[k]\\ j\in[T]}} g\left( Z^{\left( j \right)}_i \right)   \right]
            \label{eq:coupling-pi-marginals}
        \end{equation}
        Next note that
        \begin{align*}
            \Ex_{ \Pi } \left[ \sup_{g \in \G } \sum_{i = 1}^T g\left( X_i \right)    \right] & = \Ex_{\Pi}\, \left[ \mathbb{I} \left( \mathcal{E} \right) \cdot \sup_{g \in \G } \sum_{i = 1}^T g\left( X_i \right)   \right] + \Ex_{\Pi}\, \left[ \mathbb{I} \left( \mathcal{\overline{E}} \right) \cdot \sup_{g \in \G } \sum_{i = 1}^T g\left( X_i \right)   \right] \\
                & \leq T^2 \left( 1 - \sigma \right)^{\frac{\alpha}{\sigma}}  +  \Ex_{\Pi}\, \left[ \mathbb{I} \left( \mathcal{\overline{E}} \right) \cdot \sup_{g \in \G } \sum_{i = 1}^T g\left( X_i \right)   \right] \\ 
                & \leq T^2 \left( 1 - \sigma \right)^{\frac{\alpha}{\sigma}}  + \Ex_{\Pi}\, \left[ \mathbb{I} \left( \mathcal{\overline{E}} \right) \cdot \sup_{g \in \G } \sum_{i,j} g\left( Z^{\left( j \right)}_i \right)   \right] \\ 
                & \leq T^2 \left( 1 - \sigma \right)^{\frac{\alpha}{\sigma}}  + \Ex_{\Pi}\, \left[ \sup_{g \in \G } \sum_{i,j} g\left( Z^{\left( j \right)}_i \right)   \right],
      \end{align*}
where the second transition uses the fact that $\Pr\left[ \mathcal{E} \right] \leq T \left( 1 - \sigma \right)^{\frac{\alpha}{\sigma}}$ and that $ \sup_{g \in \G } \sum_{i =1}^T g\left( X_i \right)  \leq T  $ given that $\forall g \in \G, g\left( x \right) \leq 1$. The third transition uses the fact that conditioned on $\mathcal{\overline{E}}$, $ \left\{ X_1 , \dots , X_T\right\} \subseteq \left\{ Z_i^{\left( j \right)} \mid { i \in \left[ k \right] , j\in \left[ T \right] }  \right\}$. Using \autoref{eq:coupling-pi-marginals} completes the proof of \cref{claim:couple-regret}.
 \end{proof}

\begin{claim} \label{claim:oblivious-bernstein-regret}
For  any $k$ and any $\epsilon > \frac{120d \log \left( 4 e^2 / \epsilon  \right)}{Tk }   $, we have
        \begin{equation*}
 \Ex_{\U}\, \left[ \sup_{g \in \G } ~ \sum_{\substack{i \in[k], j\in[T]}} g\left( Z^{\left( j \right)}_i \right)   \right]\leq 72 \sqrt{\epsilon\, T\, k\, d \log \left( 1 / \epsilon \right) } + T\, k\, \epsilon.
        \end{equation*}
\end{claim}
\begin{proof}[Proof sketch of \cref{claim:oblivious-bernstein-regret}]
The crux of this proof is that random variables $Z_{i}^{\left( j \right)}$ are drawn i.i.d. from the uniform distribution, therefore, standard VC theory arguments provide uniform convergence bounds for them.
We use Bernstein style uniform convergence bound and leverage the fact that for all $g\in \G$, $\Ex_{\U}[g(Z)] \leq \epsilon$ to get a variance that shrinks with $\epsilon$. 
In particular, the proof of this claim follows from \cite[Theorem 13.7]{boucheron2013concentration} and is included in \cref{app:claim:oblivious-bernstein-regret} for completeness.
\end{proof}
Combining \cref{claim:couple-regret} and \cref{claim:oblivious-bernstein-regret}, replacing in values of $\alpha = 10\ln(T)$, $k = \alpha/\sigma$, and $(1-\sigma)^{\alpha/\sigma} \leq \exp (-\alpha)$, we have that
\begin{align*}
\Ex_{\adist} \left[ \sup_{g \in \G } \sum_{i = 1}^T g\left( x_i \right)  \right] 
&\leq T^2 \exp(- \alpha) + O\left(    \sqrt{\frac{\epsilon}{\sigma} T \ln(T) d \log \left( 1 / \epsilon \right) } + T \ln(T) \frac{\epsilon}{\sigma}  \right) \\
& \asedit{\leq}  O\left(    \sqrt{\frac{\epsilon}{\sigma} T \ln(T) d \log \left( 1 / \epsilon \right) } + T \ln(T) \frac{\epsilon}{\sigma}  \right),
\end{align*}
where the last transition is due to $T^2 \exp(-10 \ln(T)) \in o(1)$. This completes the proof of \cref{lem:net_bound}.
    \end{proof}

\subsection{Proof of \cref{thm:regret-main}} \label{sec:regretmain}
The proof of \cref{thm:regret-main} follows the proof outline for oblivious smooth adversaries described with \cref{sec:regret-overview} with the exception of using \cref{lem:net_bound} that holds for adaptive smooth adversaries in place of \autoref{eq:nadist-bound} bound.

Let $d = \vc(\H)$. Using the regret decomposition \autoref{eq:Regret_Decomp}, an upper bound on the size of an $\epsilon$-cover such as $|\H|\leq (41/\epsilon)^d$ (see \cite{HAUSSLER1995217} or \cite[Lemma 13.6]{boucheron2013concentration}), and \cref{lem:net_bound}, we have
\begin{align*}
\Ex[\regret(\A, \adist )] 
&\leq  O\left( \sqrt{T d \ln\left(\frac 1\epsilon \right) } \right)  + \Ex_{\adist}\left[ \sup_{g \in \G} \sum_{t = 1}^T g\left( x_t \right) \right] \\
& \leq  O\left(   \sqrt{T d \ln\left(\frac 1\epsilon \right) } +   \sqrt{\frac{\epsilon}{\sigma} T \ln(T) d \log \left( 1 / \epsilon \right) } + T \ln(T) \frac{\epsilon}{\sigma}  \right),
\end{align*}

Recall that we needed $\epsilon > \frac{120d \sigma \log \left( 4 e^2 / \epsilon  \right)}{T \log T } $. This can be satisfied by setting $ \epsilon = O \left( \frac{d \sigma }{T \log T     } \log \left(\frac{T \log T     }{d \sigma } \right)  \right)  $ and we have that 
        \begin{equation*}
            \Ex[\regret(\A, \adist )] \leq \tilde{O} \left(\sqrt{Td \ln \left( \frac{T}{d \sigma } \right) } +  d   \ln\left( \frac{T }{d \sigma } \right)   \right)
        \end{equation*}
        as required.

\subsection{Proof of \cref{thm:regret_lowerbound}} \label{sec:regretlowerproof}

In this section, we provide a proof for the tightness of our regret bounds. 
In order to do this, we first formally define the notion of Littlestone dimension of a class.

\begin{definition}[Littlestone Dimension, \cite{ben2009agnostic}]
    Let $\X$ be an instance space and $\F$ be a hypothesis class on $\X$. 
    A mistake tree is a full binary decision tree whose internal nodes are labelled by elements of $\X$. 
    For every choice of labels $ \{ y_i \}_{i=1}^d $, 
    Every root to leaf path in the mistake tree corresponds to a sequence $ \{  \left( x_i , y_i  \right) \}_{i=1}^{d}$ by associating a label $y_i$ to a node depending on whether it is the left or right child of its parent.   
    A mistake tree of depth $d$ is said to be shattered by a class $\F$ if for any root to leaf path $ \{  \left( x_i , y_i  \right) \}_{i=1}^{d}$, there is a function $f \in \F $ such that $ f\left( x_i \right) = y_i  $ for all $i \leq d$. 
    The Littlestone dimension of the class $\F$ denoted by $\lsd \left( \F \right) $ is the largest depth of a mistake tree shattered by the class $\F$. 
\end{definition}

{As an example, the Littlestone dimension of the class of thresholds on $\{1, \dots, n\}$ is $\log_2(n)$.}
The following theorem shows that the Littlestone dimension captures the regret in the online learning game against a class. 
We will only need the lower bound but we will state the full theorem for completeness.   

\begin{theorem}[\cite{ben2009agnostic,RegretTight}]\label{thm:regretLD}
    Let $\X$ be an instance space and $\F$ be a hypothesis class on $\X$. 
    Then, there exists an online learning algorithm $\mathcal{A}$ such that 
    \begin{equation*}
        \regret \left( \mathcal{A} \right) \leq O\left(  \sqrt{   \lsd \left( \F \right) T  } \right).
    \end{equation*}
    Furthermore, for any algorithm $\mathcal{A}'$, we have that
    \begin{equation*}
        \regret \left( \mathcal{A} '  \right) \geq \Omega \left(   \sqrt{  \lsd \left( \F \right) T  } \right) . 
    \end{equation*} 
\end{theorem}

Using the above theorem, we lower bound the regret in the online learning against smoothed adversaries. 
We do this by reducing the smoothed case to the worst case for a related class and lower bound the worst case regret using the above theorem. 

\begin{proof}[Proof of \cref{thm:regret_lowerbound}]
    We will first construct a class on the domain $\left[ \nicefrac{1}{\sigma} \right] = \left\{ 1, \dots, \frac{1}{\sigma} \right\}$ with VC dimension $d$ and Littlestone dimension $\Theta\left( d \log \left( \nicefrac{1}{d \sigma} \right) \right) $. 
    For simplicity, assume $ \sigma^{-1}  $ and $ d $ to be powers of two. 
    Divide $ \left[ \nicefrac{1}{\sigma} \right] $ into $d $ subsets each of equal size, denoted by $A_i$. 
    On each of these subsets instantiate the class of thresholds, i.e., for each $\gamma \in A_i $, $ h_{\gamma} \left( x \right)  = \mathbb{I} \left[ x \geq \gamma \right] $ for $x \in A_i$ and $0 $ for $x \notin A_i$. 
    For a $d$-tuple of thresholds $ \left( h_{\gamma_1} \dots h_{\gamma_d} \right) $ with $\gamma_i \in A_i $, define the function 
    \begin{equation*}
        h_{ \gamma_1, \dots, \gamma_d  } \left( x \right) = \sum_{i=1}^d \mathbb{I} \left[ x \in A_i \right] h_{\gamma_i} \left( x \right). 
    \end{equation*} 
    This function can be seen as the union of the thresholds $h_{\gamma_i}$.
    Define $\mathcal{H}$ to be the class of all such functions. 
    Note that this class has VC dimension $d$. 
    The VC dimension is at most $d$ since if any more than $d$ points would mean at least one of the $A_i$ must have two points but this cannot be shattered by thresholds on $A_i$. 
    The VC dimension can be seen to be at least $d$ by taking one point in each of the $A_i$.  

{    We claim that this class has Littlestone dimension $\Theta\left( d \log \left( \nicefrac{1}{\sigma d} \right) \right) $. At a high level, the Littlestone dimension of the class of thresholds defined over $A_i$ is $\log_2(\nicefrac{1}{\sigma d})$. Moreover, our definition of a $d$-tuple threshold is a disjoint union of $d$ thresholds. This allows us to combine the mistake trees for $A_1, \dots, A_d$, by gluing a copy of the mistake tree for $A_{i+1}$ at each of the leaves of the mistake tree for $A_i$, recursively. This results in a mistake tree of depth $\Theta\left( d \log \left( \nicefrac{1}{\sigma d} \right) \right) $. For more detail, see \cref{lem:Littlestone}.
}

    Next consider the set $\left[ 0,1 \right]$ and divide it into contiguous subintervals of length $ \sigma $. 
    We define the projection function $ \Pi : \left[ 0,1 \right] \to \left[ \nicefrac{1}{\sigma} \right] $ by $ \Pi\left( x \right) = i  $ if $x$ is in the $i$th subinterval. 
    Define the class $ \G$ on $\left[ 0,1 \right] $ by composing $ \H $ with $ \Pi $, i.e., $ \G = \left\{  g : g = h \circ \Pi \right\} $. 
    Note that the uniform distribution on each subinterval is $ \sigma $-smooth. 
    Thus, in a smoothed online learning game with the class $\G $, an adversary who plays only uniform distributions on the subintervals defined above corresponds to an adversary in the worst-case online learning game on $ \left[ \nicefrac{1}{\sigma} \right] $ against class $\H$. 
    In particular, any algorithm for $\G$ against such an adversary can be converted to an algorithm for $\H$ with the same regret. 
    From \cref{thm:regretLD}, we have that the regret against $\H$ is lower bounded by 
    \begin{equation*}
        \sqrt{ T \ldim\left( \H \right) } = \sqrt{ dT \log\left( \nicefrac{1}{\sigma d} \right)  }
    \end{equation*}
    Thus, the regret in the smoothed online learning game for $\G$ is lower bounded by $ \sqrt{ dT \log\left( \nicefrac{1}{\sigma d} \right)  }$ as required. {We note that this reduction goes through even for non-adaptive smooth adversaries.}
\end{proof}

%% file: discrepancy_2.tex
\section{Discrepancy} \label{sec:Discrepancy}

In this section, we consider the online vector balancing problem with adaptive smooth adversaries 
and achieves bounds that are almost as small as the stochastic setting where instances are drawn from the uniform distributions. 

Recall that in the online vector balancing or discrepancy problem,
at every round $t$ the algorithm see a new vector $X_t$ with bounded norm and has to assign a sign $\epsilon_t \in \{-1,1\}$ to it.  The goal of the algorithm is to ensure that for all $t\leq  T$, 
\begin{equation*}
    \norm{ \sum_{i = 1}^t \epsilon_i X_i }_{\infty}
\end{equation*}
is small.
This problem is studied under different choice of norms, but we restrict our our discussion to the infinity norm. 
In the adaptive adversarial model, where the adversary's choice of 
 vector $X_t$ could depend on the past choices of the algorithm and the adversary, i.e., $\epsilon_1, \dots, \epsilon_{t-1}$ and $X_1, \dots, X_{t-1}$, no algorithm can obtain discrepancy bound of  $O\left( \sqrt{T} \right)$.
On the other hand, recent works of \cite{Bansal_Discrepancy} and \cite{ALS_Disc} have shown that 
$\mathrm{polylog}(nT)$ discrepancy bounds are achievable when $X_t$s are drawn from a fixed distribution or are fixed by an oblivious adversary in advance.

We consider the online discrepancy problem under against an adaptive $\sigma$-smooth adversary.
That is, the adversary chooses a $\sigma$-smooth distribution for $X_t$ after having observed $\epsilon_1, \dots, \epsilon_{t-1}$ and $X_1, \dots, X_{t-1}$. We also restrict our attention to the isotropic case where the covariance matrix $\Ex_{X_t} \left[ X_t X_t^{\top} \right] = cI$ for some $c$. 

In this section, we give discrepancy bounds that smoothly interpolate between the stochastic and adaptive cases.

\begin{theorem} \label{thm:main_discrepancy}
    Let $\adist$ be an adaptive sequence of $\sigma$-smooth distributions, such that the distribution of $X_i$, with $\norm{X_i} \leq 1 $, at time $i$ is decided after observing $X_1 , \dots , X_{i-1} , \epsilon_1 ,  \dots , \epsilon_{i-1} $. {Furthermore, let $\Ex_{X_t} \left[ X_t X_t^{\top} \right] = cI $ for some some $c\in[0,1/n]$.}
    Then, there is an online algorithm for deciding the sign $\epsilon_i$ of $X_i$ such that with probability $1 - T^{-4}$ for all $t\leq T$
    \begin{equation*}
        \norm{ \sum_{i = 1}^t \epsilon_i X_i   } \leq O\left( \log^2 \left( \frac{ T n  }{ \sigma} \right) \right). 
    \end{equation*}
\end{theorem}

{We complement this upper bound by showing that
we cannot get the logarithmic dependence on smoothness parameter $\sigma $, $n$ and $T$ simultaneously without further assumptions on the distribution such as isotropy. 
}

\begin{theorem}\label{thm:disc_lowerbound}
    For any online algorithm, there is an adaptive sequence of $ \left(  \nicefrac{1}{20 n^2T^2} \right) $-smooth distributions on the unit ball such that,  we have 
    \begin{equation*}
        \norm{ \sum_{i=1}^T  \epsilon_i v_i }_{\infty} \geq \Omega \left( \sqrt{\frac{T}{n} } \right)
    \end{equation*}
    with probability $1 - \exp\left( - \nicefrac{T}{12} \right)$. 
\end{theorem}

\subsection{Overview of Existing Approaches and their Need for Obliviousness}
\label{sec:Discrepancy-bansal}

    \cite{Bansal_Discrepancy} consider various versions of the online discrepancy problem where
the vectors are chosen stochastically from a fixed {known}  distribution. One such problem is the stochastic online variant of the Komlos problem, where the input vectors come from a fixed distribution supported on the unit Euclidean ball, and the algorithms goal is to minimize the infinity norm of the discrepancy vector, i.e.,  $\| d_t\|_\infty$.
 To do this,    \cite{Bansal_Discrepancy} introduced the following potential function 
    \begin{equation*}
        \Phi_t  = \Ex_{ W \sim p }\left[  \cosh \left( \lambda d_t^{\top} W \right)\right],
    \end{equation*}
    where $p$ denotes a mixture between sampling from the fixed distribution the vectors are drawn from and the basis vectors $e_i$s. {
This potential can be seen as the exponential moment of the random variable $d_{t-1}^{\top} W$ that both bounds $ \lambda d_{t-1}^{\top}X_t \leq O \left(\log \left( T \Phi_{t-1} \right) \right) $
and induces an anti-concentration constraint on the correlations of the discrepancy
 vector $d_{t-1}$.}
\cite{Bansal_Discrepancy} then uses an algorithm that at time $t$ observes $X_t$ and picks the sign $\epsilon_t$ that minimizes the increase in the potential function $\Phi_t - \Phi_{t-1}$, that is $\Delta \Phi = \Ex_{ W \sim p }\big[  \cosh\big( \lambda (d_{t-1} +\epsilon_t X_t) ^{\top} W \big)\big] - \Ex_{ W \sim p }\big[  \cosh \big( \lambda d_{t-1}^{\top} W \big)\big]$.
    At the heart of the analysis of \cite{Bansal_Discrepancy} is to show that in expectation over the choice of$X_t$ from the fixed distribution, $\Delta\Phi$ remains small at every time step. It is not hard to see that once the expected increase in the potential is upper bounded, standard martingale techniques can be used to bound the potential and thus the discrepancy at every time step.

To bound $\Delta \Phi$, \cite{Bansal_Discrepancy} considers Taylor expansion of the potential function as follows
    \begin{equation}
        \Delta \Phi \lesssim  \epsilon_t \lambda \Ex_{ W \sim p }\left[ \sinh \left( \lambda d_{t - 1}^{\top }  W   \right) X_t^{\top } W \right] + \lambda^2  \Ex_{ W \sim p }\left[\,  \abs{\sinh \left( \lambda d_{t - 1}^{\top }  W   \right)} \cdot  W^{\top} X_t X_t^{\top} W \right]. 
        \label{eq:Bansal-decomp}
    \end{equation} 
\cite{Bansal_Discrepancy} leverages the the obliviousness of the adversary, i.e., the fact that $X_t$ arrive from a fixed distribution, and isotropy of $X$ to directly bound the linear and quadratic terms of the Taylor expansion as follows.

The second term of \autoref{eq:Bansal-decomp} is bounded using the isotropy of the vector $X_t$ as follows
    \begin{equation*}
        \lambda^2  \Ex_{ W \sim p }  \abs{\sinh \left( \lambda d_{t - 1}^{\top }  W   \right)} W^{\top} X_t X_t^{\top} W  \leq \frac{1}{n }   \lambda^2  \Ex_{ W \sim p }  \abs{\sinh \left( \lambda d_{t - 1}^{\top }  W   \right)}. 
    \end{equation*}

As for the first term of \autoref{eq:Bansal-decomp}, note that since the algorithm picks $\epsilon_t$ to minimize the potential rise, it is sufficient to upper bound 
$         \Ex_{X_t} \left[ -  \big| \lambda \Ex_{ W \sim p } \big[ \sinh \big( \lambda d_{t - 1}^{\top }  W   \big) X_t^{\top } W \big] \big|  \right].
$
    Since the potential is the exponential moment of the $\lambda d_{t-1}^{\top}X_t $ and $X_t$s are drawn from an oblivious distribution, we have that $ \lambda d_{t-1}^{\top}X_t \leq O \left(\log \left( T \Phi_{t-1} \right) \right) $ with high probability. Thus, we get 
    \begin{align*}
        \Ex_{X_t} \left[\,  \abs{\lambda \Ex_{ W \sim p }  \sinh \left( \lambda d_{t - 1}^{\top }  W   \right) X_t^{\top } W}\right] & \gtrsim \frac{1}{\ln(T\Phi_{t-1 })} \Ex_{X_t} \left[  \lambda^2 \Ex_{ W \sim p } \left[  \sinh \left( \lambda d_{t - 1}^{\top }  W   \right)  d_{t-1}^{\top} X_t X_t^{\top } W \right]  \right] \\ 
        & \gtrsim \frac{ \lambda }{n  \ln(T \Phi_{t-1 })} \Ex_{ W \sim p }\left[\, \abs{\sinh \left( \lambda d_{t - 1}^{\top }  W   \right)} - 2 \right]. 
    \end{align*}
    using the fact that $ a \sinh (a) \geq \abs{\sinh \left( a \right)} - 2 $ and the isotopy of the distribution. 
    Summing these two terms, we get 
    \begin{align*}
        \Delta \Phi & \lesssim  - \frac{ \lambda }{n  \ln(T \Phi_{t-1 })} \Ex_{ W \sim p }\left[\, \abs{\sinh \left( \lambda d_{t - 1}^{\top }  W   \right)} - 2 \right] + \frac{1}{n }   \lambda^2  \Ex_{ W \sim p }  \abs{\sinh \left( \lambda d_{t - 1}^{\top }  W   \right)} \\ 
        & \lesssim 2. 
    \end{align*}
    We get we choose $\lambda$ such that $\lambda^{-1} \leq \log\left( T \Phi_{t-1} \right) $ if $\Phi \leq \poly\left( T \right) $. 
    This tells us that that if the potential is small, then the change in the potential is small as required.

Let us now review the steps where the obliviousness of the adversary was crucial for the analysis of \cite{Bansal_Discrepancy}. The main step is the definition and the interpretation of the potential function, that controls the moments of $d_{t-1}^{\top} X_t $ assuming that $X_t$ comes from a fixed distribution and the future vector that are represented in $W\sim p$.  That is, obliviousness is primarily used to show that $\lambda d_{t-1}^\top X_t \leq O(\ln(T \Phi_{t-1}))$. 
In an adaptive (smooth) setting where the distribution of $X_t$ and the future vectors differ and are unknown an adversary can correlate $X_t$ and the future vectors with $d_{t-1}$. 
It is not immediately clear how to directly adapt the potential function to account for the an evolving sequence of distributions.
A possible approach for directly altering the potential function is to work with worst-case evolution of smooth distribution across a single time step. This seems both algorithmically challenging to deal with and as we see next unnecessary.

    \subsection{From Adaptive to Oblivious through Coupling}
\label{sec:discrepency:smooth}    
    
    We emphasize that the main step in which \cite{Bansal_Discrepancy} leveraged the obliviousness of the adversary is to show that their potential function defined over random $X_t$ and a random $W\sim p$ that balances between future observations and the standard basis 
    has the property that $\lambda d_{t-1}^\top X_t \leq O(\ln(T \Phi_{t-1}))$.   
We use the coupling argument to show that a similarly defined potential function in our case also demonstrate the same bounds.  
The main 
observation that allows us to move from the oblivious adversary to the adaptive adversary is that the coupling discussed in \cref{sec:coupling_overview} gives us a way to upper bound the probability that 
$d_{t-1}^\top X_t$ is large under an adaptive sequence of smooth distributions in terms of the probability under the uniform distribution.

Let us start by defining the algorithm that obtains our results of \cref{thm:main_discrepancy} analogously to the algorithm of \cite{Bansal_Discrepancy} for the uniform distribution. 
At step $t$, our algorithm observes vectors the discrepancy vector $d_{t-1}$ (which is a function 
of $\epsilon_{1}\dots, \epsilon_{t-1}$ and the previous vectors) and receives a new vector $X_t$ that is to be colored. Let $\epsilon_t$ denote the sign that our algorithm will assign to $X_t$ and let $d_t = d_{t-1} + \epsilon_t X_t$. Let $p$ denote the following distribution. 
        \begin{equation*}
            \begin{cases} 
            Z \sim \mathcal{U} & \text{ with probability } \frac{1}{2} \\ 
             e_i  \text{ where } e_i \sim p_{\mathrm{basis}} & \text{ with probability } \frac{1}{2}
            \end{cases}, 
        \end{equation*}
        where $p_{\mathrm{basis}}$ is the uniform distribution on the standard basis vectors (with both positive and negative signs). Defined the potential function 
        \begin{equation*}
            \Phi_t = \Ex_{ W \sim p }\left[\cosh \left( \lambda d_t^{\top} W \right) \right],
        \end{equation*}
        for $\lambda = 1000 \ln\left( knT \right) $ where $k$ is a parameter to be set later. 
        At step $t$ observing $X_t$ our algorithm greedily picks the $\epsilon_t$ minimizes the potential difference, that is
        \[ \Phi_t - \Phi_{t-1} = \Ex_{W\sim p}\left[\cosh \left( \lambda (d_{t-1} + \epsilon_t X_t)^{\top} W \right) \right]  - \Ex_{W\sim p} \left[\cosh \left( \lambda d_{t-1}^{\top} W \right) \right].\]

The following lemma uses the coupling argument to bound the probability tails of $d_{t-1}^\top X_t$.

        \begin{lemma} \label{lem:prob_upper}
           Consider any fixed $d_{t-1}$ vector and $X_t$ that is  sampled from an arbitrary $\sigma$-smooth distribution.             Then, 
                        \begin{equation*}
                \Pr_{X_t}\left[ \lambda d_{t-1}^{\top} X_t \geq 4  \ln \left( \frac{4 k \Phi_{t-1}}{\delta} \right) \right] \leq \left( 1 - \sigma \right)^k + \delta . 
            \end{equation*}
        \end{lemma} 
        \begin{proof}
            We will use the coupling from \cref{sec:main_coupling}. 
            In particular, we can use a single-step coupling from \cref{lem:single_round} that shows that there exists a coupling $\Pi$ on $\left( \tilde{X_t} , Z_1^{\left( t \right)} , \dots , Z_k^{\left( t \right)}  \right)$ such that $\tilde{X_t}$ has the same distribution as $X_t$, $Z_1^{\left( t \right)} , \dots , Z_k^{\left( t \right)}$ are uniformly and independently distributed and with probability at most $ \left( 1 - \sigma \right)^k $, we have $ \tilde{X_t}  \notin \left\{  Z_1^{\left( t \right)} , \dots , Z_k^{\left( t \right)} \right\}   $. 
            Let $\mathcal{E}$ denote the event  where $ \tilde{X_t}  \notin \left\{  Z_1^{\left( t \right)} , \dots , Z_k^{\left( t \right)} \right\}   $. Then, for any $\theta$ 
            \begin{align*}
                \Pr_{X_t}\left[ \lambda d_{t-1}^{\top} X_t \geq \theta \right] & = \Pr\left[ \exp\left(\lambda d_{t-1}^{\top} X_t \right) \geq \exp(\theta)  \right] \\
                & = \Pr_{\Pi} \left[ \mathcal{E} \land \left\{ \exp\left( \lambda d_{t-1}^{\top} \tilde{X}_t  \right)\geq \exp(\theta) \right\}  \right] + \Pr_{\Pi}\left[ \overline{\mathcal{E} } \land  \left\{ \exp\left(\lambda d_{t-1}^{\top} \tilde{X}_t \right) \geq \exp(\theta) \right\} \right] \\ 
                & \leq  \left( 1 - \sigma \right)^{k} + \Pr_{\Pi}\left[ \overline{\mathcal{E} } \land \left\{ \sum_{i=1}^{k} \exp\left( {\lambda d_{t-1}^{\top} Z_{i}^{\left( t \right)} } \right) \geq \exp(\theta)  \right\} \right] \\ 
                & \leq  \left( 1 - \sigma \right)^{k} + \Pr_{\Pi}\left[  \sum_{i=1}^{k} \exp\left({\lambda d_{t-1}^{\top} Z_{i}^{\left( t \right)} } \right) \geq \exp(\theta) \right] \\ 
                & \leq \left( 1 - \sigma \right)^{k} + \exp(-\theta) \Ex_{\Pi}\left[ \sum_{i=1}^{k} \exp\left({\lambda d_{t-1}^{\top} Z_{i}^{\left( t \right)} } \right)  \right] \qquad \text{(By Markov inequality)}\\ 
                & \leq \left( 1 - \sigma \right)^{k} + 2 \exp(-\theta)   \Ex_{\Pi}\left[ \sum_{i=1}^{k} \cosh\left( \lambda d_{t-1}^{\top} Z_{i}^{(t)} \right)  \right]   \qquad  \text{(By} \exp(x)\leq 2\cosh (x) \text{)}\\                
                & \leq \left( 1 - \sigma \right)^{k} + 4 \exp(-\theta)   \sum_{i=1}^{k} \Ex_{W\sim p}\left[ \cosh\left( \lambda d_{t-1}^{\top} W \right)  \right]   \qquad  \text{($p$ is w.p. $0.5$ uniform)} 
              \\                
                & \leq \left( 1 - \sigma \right)^{k} + 4 k \Phi_{t-1} \exp(-\theta),
            \end{align*}

            Setting $\theta = \ln\left( \frac{4k \Phi_{t-1}}{\delta} \right)$ completes the proof.
        \end{proof}

%% file: discrepancy_proof.tex
\subsection{Proof of \cref{thm:main_discrepancy}}
Our proof follows the same approach as that of ~\cite{Bansal_Discrepancy} outlined in \cref{sec:Discrepancy-bansal} and aims to bound $\Ex_{X_t}[\Phi_t] - \Phi_{t-1}$ at every time step. The main technical challenge is to upperbound the linear term $\Ex_{X_t}[-|L(X_t)|]$ in $\Delta \Phi_t$ as a function of the correlation between $d_{t-1}$ and $X_t$ drawn from a $\sigma$-smooth distribution. We then use our \cref{lem:prob_upper} that controls this correlation to bound the linear term. 

Recall from \cref{sec:discrepency:smooth} that our algorithm
 observes $X_t$ and picks the $\epsilon_t$ that minimizes the potential difference, that is
        \[ \Phi_t - \Phi_{t-1} = \Ex_{W\sim p}\left[\cosh \left( \lambda (d_{t-1} + \epsilon_t X_t)^{\top} W \right) \right]  - \Ex_{W\sim p} \left[\cosh \left( \lambda d_{t-1}^{\top} W \right) \right].\]
        
The next lemma shows that when the potential at time $t-1$ is small, the expected increase in $\Phi_t$ over the choice of $X_t$ is small.  

        \begin{lemma} \label{lem:potential}
            At any time $t$, if $\Phi_{t-1} \leq T^6 $, then $ \Ex_{X_t } \left[ \Phi_t \right] - \Phi_{t-1} \leq 2  $. 
        \end{lemma}
        \begin{proof}
            Denote $ \Delta \Phi = \Phi_t - \Phi_{t-1}  $. 
            As in \cite{Bansal_Discrepancy}, we decompose this as 
            \begin{align*}
&                \Delta \Phi \left( X_t \right) = \Ex_{ W \sim p } \left[  \cosh \left( \lambda (d_{t - 1}^{\top} + \epsilon_t X_t ) W \right) \right] - \Ex_{ W \sim p } \left[ \cosh \left( \lambda (d_{t - 1}^{\top}  ) W \right) \right] \\ 
                &  \leq \epsilon_t \lambda \Ex_{ W \sim p }  \left[ \sinh \left( \lambda d_{t - 1}^{\top }  W   \right) X_t^{\top } W \right]  + \lambda^2  \Ex_{ W \sim p } \left[\  \abs{\sinh \left( \lambda d_{t - 1}^{\top }  W   \right)} W^{\top} X_t X_t^{\top} W  \right] +  \lambda^2  \Ex_{ W \sim p }   \left[ W^{\top} X_t X_t^{\top} W\right]. 
            \end{align*}
            Using notation similar to \cite{Bansal_Discrepancy}, we will denote the first term in last equation as $\epsilon_t L\left( X_t \right) $, the second as $Q \left( X_t \right)$ and the third as $Q_{*} \left( X_t \right)$. 
            We need to upper bound $ \Ex_{X_t} \left[ \Delta \Phi(X_t) \right] $ and thus it suffices to bound these three quantities.

Our approach for upper bounding            $\Ex_{X_t} \left[ Q \left( X_t \right)  \right] $ and $\Ex_{X_t} \left[ Q_{*} \left( X_t \right)  \right]$ is similar to \cite{Bansal_Discrepancy} and uses that fact that the distribution of $X$ is isotropic (without the need to bring in smoothness). 
We state these bounds in the following claim and include the proof of them for completeness in \cref{sec:miss_proof_disc}. 

\begin{claim}\label{claim:lem:quadratics}
                Let $Q$ and $Q_{*}$ be defined as above. 
                Then, 
                \begin{equation*}
                    \Ex_{X_t} \left[ Q \left( X_t \right)  \right] \leq   c \lambda^2  \Ex_{ W \sim p } \left[\  \abs{\sinh \left( \lambda d_{t - 1}^{\top }  W   \right)}\ \right]
          \quad \text{ and } \quad 
                    \Ex_{X_t} \left[ Q_{*} \left( X_t \right)  \right] \leq  \frac{ c\lambda^2 }{n}. 
                \end{equation*}
            \end{claim}
           To upper bound $\Ex[\epsilon_tL(X_t)]$, we need to use both the smoothness of $X_t$ and their isotropic nature. 
First note that since $\epsilon_t$ is chosen to minimize the potential drop, we can bound $\Ex_{X_t} \left[ \epsilon_t L\left( X_t \right) \right] \leq -  \Ex_{X_t} \left[ \ \abs{L\left( X_t \right) }\right]$. So it's sufficient to lower bound $\Ex_{X_t} \left[\abs{L(X_t)} \right]$.
            \begin{claim} \label{claim:lem:linear}
Let $L$  be defined as above. 
                Then,              
                   \begin{equation*}
                     \Ex_{X_t} \left[\  \abs{L\left( X_t \right) }\right] \geq  \frac{c\lambda}{\ln\left( 4k \Phi_{t-1} / \delta \right)} \Ex_{ W \sim p }\left[\  \abs{ \sinh \left( \lambda d_{t - 1}^{\top }  W   \right)  } \right] - 1
                \end{equation*}
            \end{claim}
            \begin{proof}[Proof of \cref{claim:lem:linear}]
  Let $B =  \ln\left( 4k \Phi_{t-1} / \delta \right)$ and let $G  $ be the event that $\lambda \abs{d_{t-1}^{\top} X_t  } \leq B$.                              
                  Note that $\abs{L\left( X_t \right)} \geq L\left( X_t \right) \cdot f\left( X_t \right) / \norm{f}_{\infty}  $ for any function $f$.                
                We will use the function $f \left( X_t \right) = d_{t-1}^{\top} X_t \cdot \mathbb{I} \left[ X_t \in G \right] $ and note that $\norm{f}_\infty\leq B/\lambda$.
               This allows us to decompose $ \abs{L}$ further as follows.
                    \begin{align*}
 &                   \Ex_{X_t} \left[\, \abs{L \left( X_t \right) } \right] \geq \Ex_{X_t} \left[  \frac{\lambda^2}{B } \Ex_{ W \sim p } \left[ \sinh \left( \lambda d_{t - 1}^{\top }  W   \right) d_{t - 1}^{\top} X_t X_t^{\top} W \cdot \mathbb{I} \left( X_t \in G \right) \right] \right] \\ 
                    &=   \frac{\lambda^2}{B } \Ex_{ W \sim p }\left[  \sinh \left( \lambda d_{t - 1}^{\top }  W   \right) d_{t - 1}^{\top} \Ex_{X_t} [X_t X_t^{\top}] W \right] -  \frac{\lambda^2}{B } \Ex_{ W \sim p } \left[ \sinh \left( \lambda d_{t - 1}^{\top }  W   \right) d_{t - 1}^{\top}  \Ex_{X_t} \left[ X_t X_t^{\top} \mathbb{I} \left( X_t \notin G \right)  \right]  W \right]. 
                \end{align*}
                Looking at the second term in the above equation and using the fact that $X$ is an isotropic distribution and \cref{lem:prob_upper} (which used the smoothness of $X$), we have 
                \begin{equation*}
                   \norm{ \Ex_{X_t} \left[ X_t X_t^{\top} \mathbb{I} \left( X_t \notin G \right)   \right] }_{\mathrm{op}} \leq \Pr\left[ X_t \notin G  \right] \leq \left( 1 - \sigma \right)^k +  \delta. 
                \end{equation*}
                Ensuring that $k >> \sigma^{-1} \ln \left( 1 /\delta  \right) $ by $k = 100 \sigma^{-1} \ln \left( T \ln \left( T \right) \right) $ and noting that $\norm{d_{t-1}} \leq T $
                \begin{equation*}
                    d_{t - 1}^{\top}  \Ex_{X_t} \left[ X_t X_t^{\top} \mathbb{I} \left( X_t \notin G \right)  \right]  W \leq 2 \delta T . 
                \end{equation*}
                Picking $\delta^{-1} = 2\lambda \Phi_{t-1} T  $, we get
                 \begin{equation*}
                   \lambda \abs{ d_{t - 1}^{\top}  \Ex_{X_t} \left[ X_t X_t^{\top} \mathbb{I} \left( X_t \notin G \right) \right]  W } \leq \Phi_{t-1}^{-1} . 
                 \end{equation*}
                Now let us consider the first term of the above decomposition. Using the fact that $X$ is an isotropic random variable, we have 
                \begin{align*}
                    \frac{\lambda^2}{B } \Ex_{ W \sim p } \left[ \sinh \left( \lambda d_{t - 1}^{\top }  W   \right) d_{t - 1}^{\top} \Ex_{X_t} [X_t X_t^{\top}] W \right] & = \frac{ c\lambda}{B } \Ex_{ W \sim p }\left[ \sinh \left( \lambda d_{t - 1}^{\top }  W   \right)  \lambda d_{t - 1}^{\top} W \right] \\ 
                    & \geq \frac{ c\lambda}{B } \Ex_{ W \sim p }\left[\,  \abs{ \sinh \left( \lambda d_{t - 1}^{\top }  W   \right)  } -2 \right],
                \end{align*}
                where the last inequality used the fact that $ a\sinh\left( a \right) \geq \abs{\sinh(a)} -2 $. 
                Putting the inequalities together, we get 
                \begin{align*}
                    \Ex_{X_t} \left[  \abs{L\left( X_t \right) }\right]   & \geq \frac{c\lambda}{B  } \Ex_{ W \sim p } \left[\,  \abs{ \sinh \left( \lambda d_{t - 1}^{\top }  W   \right)  }  -2\right]  - \frac{c\lambda}{B} \Phi_{t-1}^{-1} \Ex_{ W \sim p } \left[\,  \abs{\sinh \left( \lambda d_{t - 1}^{\top }  W   \right) }  \right] \\ 
                    & \geq \frac{c\lambda}{B } \Ex_{ W \sim p } \left[\, \abs{ \sinh \left( \lambda d_{t - 1}^{\top }  W   \right)  }\right] - \frac{2 c \lambda}{B}  - \frac{\lambda}{B} \\ 
                    & \geq \frac{c \lambda}{B } \Ex_{ W \sim p }\left[\,  \abs{ \sinh \left( \lambda d_{t - 1}^{\top }  W   \right)  } \right] - 1,
                \end{align*}
              where the second transition is by the definition of $\Phi_{t-1}$ and the third transition is by the values of $\lambda^{-1} = 1000\ln(knT)$, $B = \ln(8\lambda k T \Phi^2_{t-1})$, and the assumption that $\Phi_{t-1}\leq T^6$. This completes the proof of \cref{claim:lem:linear}.

            \end{proof}
            We now use \cref{claim:lem:quadratics} and \cref{claim:lem:quadratics} to finish the proof of \cref{lem:potential} as follows
            \begin{align*}
                \Ex_{X_t} \left[ \Delta \Phi \left( X_t \right)  \right] & \leq \Ex_{X_t} \left[ - \abs{L} + Q + Q_{*} \right] \\ 
                & \leq -\frac{c\lambda}{B } \Ex_{ W \sim p }\left[\,  \abs{ \sinh \left( \lambda d_{t - 1}^{\top }  W   \right)  } \right]+ 1 +   c \lambda^2  \Ex_{ W \sim p } \left[\, \abs{\sinh \left( \lambda d_{t - 1}^{\top }  W   \right)} \right]+ \frac{c\lambda^2 }{n} \\ 
                & \leq 2 
            \end{align*}
            Here, we use the fact that $\lambda \leq B^{-1} $ which {follows from
 $\lambda^{-1} = 1000\ln(knT)$, $B = \ln(8\lambda k T \Phi^2_{t-1})$, and the assumption that $\Phi_{t-1}\leq T^6$.} This completes the proof of \cref{lem:potential}.
   \end{proof}
        
        Note that the above argument gives us $ \Ex_{X_t} \left[ \Delta \Phi | \Phi_{t-1}  \right] \leq 2 $ given that $ \Phi_{t-1} \leq  T^6 $. 
        We truncate $\Phi_t$ at $ T^6 $, i.e. setting $ \tilde{\Phi_t} = \Phi_t $ till $\Phi_t \leq T^6 $ and $ \tilde{\Phi}_t = T^6 $ afterwards.
        Using this and the Doob maximal martingale inequality, it follows that $ \Phi_t \leq T^6 $ with probability $1 - T^4$ as required.  

        Next, we will see why bounding the potential suffices to bound the discrepancy. 
        Recall that the potential was defined as $\Phi_t = \Ex_{ W \sim p }\left[ \cosh \left( \lambda d_t^{\top} W_i \right) \right]$.
        Since with probability $1/2 $, $p$ samples uniformly from the set of basis vectors $p_{basis}$ and given that $\exp(x) \leq 2 \cosh(x) $, we have 
  $  \exp\left( \lambda \abs{ d_{t}^{\top} {e_i}  } \right)  \leq \sum_{i=1}^n \exp\left( \lambda \abs{ d_{t}^{\top} {e_j}  } \right) \leq 8 n \Phi_t $ for all basis  vectors $e_j$. 
        Thus, we have 
        \begin{equation*}
            \norm{ d_t   }_{\infty} = \norm{ \sum_{i=1}^t \epsilon_i X_i  } \leq \lambda^{-1} \ln\left( 4n \Phi_t  \right). 
        \end{equation*}
        Recall that  $\lambda^{-1} = 1000 \ln \left( \frac{nT \ln(T)}{\sigma} \right)$, which gives us that 
        \begin{equation*}
            \norm{ \sum_{i=1}^t \epsilon_i X_i  } \leq \tilde{O} \left(  \ln^2 \left( \frac{nT}{\sigma} \right)  \right)
        \end{equation*}
        as required.

        \subsection{Proof of \cref{thm:disc_lowerbound}}

            Here, we show that the isotropy condition is required for our online discrepancy upper bound.        Recall that the worst-case adversary for discrepancy generated vectors that were orthogonal to the current discrepancy vector at each time. 
            The idea for this proof is that even with the smoothness requirements, the adversary can generate vectors such that the inner products are concentrated near zero, leading to high discrepancy. 
            Let the discrepancy vector at time $t$ be denoted by $d_t$. 
            Consider the set $S_t = \left\{ x : \norm{x}_2 \leq 1, \abs{\ip{x}{ {d_{t-1}} } } \leq n^{-2} T^{-2} \norm{d_{t-1}}_2   \right\}$. 
            Note that the uniform distribution on $S_t$ is $ c n^{-2}T^{-2}  $ smooth for some constant $c$.  
            To see this, let $ \mathcal{U} $ denote the uniform distribution on the unit ball and let $V_n$ denote the volume of the unit ball in $n$ dimensions. 
            Then, 
            \begin{align*}
                \Pr_{ X \sim \mathcal{U} } \left[ X \in S_t \right] &= \frac{1}{V_n} \int_{-n^{-2}T^{-2} }^{n^{-2} T^{-2}} \left( 1 - x^2 \right)^{\frac{n-1}{2}} V_{n-1} dx \\ 
                & \geq \frac{1}{V_n} \int_{-n^{-2}T^{-2} }^{n^{-2} T^{-2}} \left( 1 - \frac{1}{n^4 T^4}  \right)^{\frac{n-1}{2}} V_{n-1} dx \\ 
                & \geq \frac{V_{n-1}}{V_n} \cdot  \frac{1}{ 2 n^{2}T^2 } \\ 
                & \geq \frac{1}{ 20 n^{2}T^2 }. 
            \end{align*}
            The second inequality follows by noting that $ \left( 1 - n^{-4}T^{-4} \right)^{\frac{n-1}{2}}  \geq \nicefrac{1}{4} $. 
            With this, we describe the adversary's strategy. 
            At time $t$, the adversary picks $v_t$ uniformly from $S_t$. 
            We will measure the squared $2$-norm of the discrepancy vector. 
            \begin{align*}
                \norm{d_{t}}^2_2 &= \norm{ \epsilon_t v_t + d_{t-1} }^2_2 \\ 
                & = \epsilon_t^2 \norm{v_t}^2_2 + \norm{d_{t-1}}^2_2 + 2 \ip{v_t}{d_t} \\ 
                & \geq \norm{v_t}^2_2 +   \norm{d_{t-1}}^2_2 - \frac{2 \norm{d_{t-1}}_2  }{n^2T^2} \\ 
                & \geq \norm{v_t}^2_2 +\norm{d_{t-1} }^2_2  - \frac{2}{ n^2 T } \\ 
                & \geq \sum_{i=1}^t \norm{v_i}_2^2  - \frac{2t}{n^2 T} . 
            \end{align*}
            Note that $ \Pr\left[ \norm{v_i}_2 \leq \nicefrac{1}{2}  \right] \leq 2^{-\left( n-1 \right)} $. 
            This can be seen by noting that the probability can be computed with an integral similar to the one above but with ball of radius $\nicefrac{1}{2}$ instead of the ball of radius $1$.
            Also, note that the lengths $ \norm{v_i}_2 $ are independent across $i$ {(even though $v_i$ themselves are not independent).}
             Denote $z_i$ as a random variable which is $1$ if $ \norm{v_i} \geq \nicefrac{1}{2} $ and $0$ otherwise. 
             Then, 
            \begin{equation*}
                 \sum_{i=1}^t \norm{v_i}_2^2 \geq \frac{1}{4} \sum_{i=1}^t z_i . 
             \end{equation*}
            Applying a Chernoff bound to $z_i$, we get 
            \begin{equation*}
                \Pr\left[ \sum_{i=1}^t \norm{v_i}^2_2 \leq \frac{t}{8} \left( 1 - 2^{-\left( d-1 \right)  }\right) \right] \leq e^{ - \frac{  t  }{12} }. 
            \end{equation*}
            Thus  with probability $1 - e^{- \frac{t}{12}}$, 
            \begin{equation*}
                \norm{d_t}^2_2 \geq \frac{t}{ 16} - \frac{2t}{n^2T } \geq \frac{t}{20}. 
            \end{equation*}
            We get the desired result by relating the $2$-norm and $\infty$-norm. 

        This shows that we cannot get the logarithmic dependence on smoothness parameter $\sigma $, $n$ and $T$ simultaneously without further assumptions on the distribution such as isotropy. 

%% file: Dispersion.tex
\section{Adaptive Smooth Adversaries and Dispersed Sequences}

In this section, we consider the problem of online optimization and show that adaptive smooth adversaries create \emph{dispersed} sequences.
Recall that in the online optimization settings, an adversary chooses a sequence of functions $u_1, \dots, u_T$ such that $u_t : \X \to \left[0, 1\right]$ and the learner responds by taking instances $x_1, \dots, x_T \in \X$ with a goal of minimizing the regret.
The main theorem of this section shows that when $u_i$s are piecewise Lipschitz functions and are chosen by an adaptive smooth adversary in such a way that the discontinuities of these functions are is smoothed, the resulting sequence of functions is dispersed.

\begin{theorem}[Adaptive Smoothness leads to Dispersion] \label{thm:ada_disp}
    Let $u_1, \dots, u_T$ be functions from $\left[ 0,1 \right] \to \br $ that are piecewise Lipschitz with $\ell$ discontinuities each. 
    Let $d_{i,j}$ denote the discontinuities of $u_i$ and that are sampled from an adaptive sequence of $\sigma$-smooth distributions.  
    Then, for any $\alpha \geq 0.5 $, with probability $1-\delta$ the sequence of functions $u_1 \dots u_T $ is $(w, k)$-dispersed for  
    \begin{equation*}
        w =  \sigma(T \ell)^{ \alpha - 1}  ~ \text{ and } ~ k =  \tilde{O}\left( \left( T \ell  \right)^{\alpha } \ln\left( \frac{1}{\delta}\right)  + \ln\left(\frac 1\sigma\right) \right)  .
    \end{equation*}
\end{theorem}

\subsection{Overview of \cite{Dispersion} and the need for Obliviousness}

\cite[Lemma 13]{Dispersion} showed a similar result to \cref{thm:ada_disp} but for sequences that are generated by an oblivious smooth adversary.
The crux of their argument is showing that for the number of points that can lie in any ball of small radius is small when these points are drawn \emph{independently} from a non-adaptive sequence of $\sigma$-smooth distributions. More formally,  they show that $\ell$ points are picked from a non-adaptive sequence of $\sigma$-smooth distributions over $\left[ 0,1 \right]$, then with probability $1-\delta$, 
 any interval of width $w$ contains at most 
\begin{equation}
    O\left(\frac{T \ell w}{ \sigma} + \sqrt{T \ell \log \left( \frac{1}{\delta} \right)} \right)
    \label{eq:dispersion-balcan}
\end{equation}
points. Setting $w = \sigma (T \ell)^{\alpha -1}$ for an $\alpha \geq 0.5$ then \cite{Dispersion} showed that  for a \emph{non-adaptive} smooth adversary, with probability $1-\delta$, $u_1 \dots u_T $ is  $\left( \sigma (T \ell)^{\alpha -1},  O\big( (T \ell)^{\alpha} \ln(\frac{1}{\delta}) \big)\right)$-dispersed.

The only step in the existing analysis that requires the adversary to be non-adaptive is that of proving \autoref{eq:dispersion-balcan}. Here, \cite{Dispersion} relies on the obliviousness of
the adversary an uses the fact that  points drawn from a non-adaptive sequence of smooth distributions are independently (but not identically) distributed.
Their approach leverages this independence between the instances and the fact that VC dimension of intervals is $2$ to use the double sampling and symmetrization tricks from VC theory and establish a uniform convergence property on the number of instances that can fall in any interval of width $w$.

\subsection{Reducing Adaptivity to Obliviousness for Dispersion via the Coupling} \label{sec:dispersion}
We emphasize that \autoref{eq:dispersion-balcan} is the only step in the existing approach that relies on the obliviousness of the adversary. In this section, we show how the coupling lemma can be used to obtain (almost) the same upper bound as of \autoref{eq:dispersion-balcan} for adaptive adversaries. 
Our approach is essentially the same as the proof of \cref{lem:net_bound} used for regret minimization, where we had to bound the expected maximum number of smooth adaptive instances that can fall in any function $g\in \G$ of bounded VC dimension. In this case, we can apply the same results to the class of intervals, which has a VC dimension of $2$, and bound the number of discontinuities than fall in any interval. We make another small change to our previous approach to achieve high probability bounds instead of bounds on the expectation. 

\begin{lemma}
\label{lem:dispersion-instances}
Let $\mathcal{J}$ be the set of all intervals of width at most $w$ over $[0,1]$.
For $i\in [T]$ and $j\in[\ell]$, let $d_{i,j}$ be drawn from a $T\ell$-step  adaptive sequence of $\sigma$-smooth random variables over $[0,1]$. 
    Then, with probability $1-\delta$,
    \begin{equation*}
 \max_{J\in \mathcal{J}} \sum_{\substack{i\in[T]\\j\in[\ell]}} \mathbb{I}  \left[ d_{i,j} \in J \right] <  \frac{T \ell w}{\sigma} \ln \left( \frac{2 T \ell }{ \delta} \right)  + 10\sqrt{ \frac{T \ell w }{ \sigma} \ln \left( \frac{2 T \ell  }{ \delta} \right)  \ln \left( \frac{1}{\delta}\right)  } + 10\log\left( \frac{10 T \ell \log \left( 2T \ell /\delta \right) }{\sigma \delta} \right) 
    \end{equation*}
\end{lemma}

\begin{proof}
Let $\adist$ represent the $T\ell$-step  adaptive sequence of $\sigma$-smooth distributions from which $d_{i,j}$s are drawn. 
    Let $k= \frac{\ln(2T\ell/\delta)}{\sigma}$  and 
    consider the coupling $\Pi$ described in \cref{sec:Coupling} over $\left(d_{i,j }, Z_{1}^{\left(i,j\right)} \dots Z_{k }^{\left( i,j \right)}  \right)_{i\in[T], j\in [\ell]}$, where $d_{i,j}$s are distributed according to $\adist$ and $Z_{m}^{(i,j)}$s are distributed according to the uniform distribution over $[0,1]$. 
 Let $\mathcal{E}$ be the event $ \left\{ d_{i , j } \mid \forall i\in [T], j\in [\ell] \right\} \not\subseteq \left\{ Z_{m}^{\left( i,j \right)} \mid \forall m\in [k], i\in[T], j\in [\ell]\right\} $. By \cref{thm:main_coupling_overview}, $\Pr[\mathcal{E}]\leq T\ell(1-\sigma)^{k}$. 

    We now bound the probability that the number of instances $d_{i,j}$s that fall in any interval of size $w$ is bigger than a threshold $\theta$, using the coupling argument.
We have    
    \begin{align*}
        \Pr_\adist\left[ \max_{J\in \mathcal{J}} \sum_{i,j} \mathbb{I}  \left[ d_{i,j} \in J \right] \geq \theta \right] &= \Pr_{\Pi}\left[ \max_{J\in \mathcal{J}} \sum_{i,j} \mathbb{I} \left[ d_{i,j} \in J \right] \geq \theta  \right] \\ 
        & = \Pr_{\Pi}\left[ \mathcal{E} \land \max_{J\in \mathcal{J}} \sum_{i,j} \mathbb{I} \left[ d_{i,j} \in J \right] \geq \theta \right] + \Pr_{\Pi}\left[ \overline{\mathcal{E}} \land \max_{J\in \mathcal{J}}  \sum_{i,j} \mathbb{I} \left[ d_{i,j} \in J \right] \geq \theta  \right] \\
        & \leq T \ell \left( 1 - \sigma \right)^{k} +  \Pr_{\Pi}\left[ \overline{\mathcal{E}} \land \max_{J\in \mathcal{J}}  \sum_{i,j} \mathbb{I} \left[ d_{i,j} \in J \right] \geq \theta  \right] \\ 
        & \leq T \ell \left( 1 - \sigma \right)^{k} + \Pr_{\Pi}\left[ \overline{\mathcal{E}} \land \max_{J\in \mathcal{J}}  \sum_{i,j,m} \mathbb{I} \left[ Z^{(i,j)}_{m} \in J \right] \geq \theta \right] \\ 
        & \leq T \ell \left( 1 - \sigma \right)^{k} + \Pr_{\Pi}\left[  \max_{J\in \mathcal{J}}  \sum_{i,j,m} \mathbb{I} \left[Z^{(i,j)}_{m} \in J \right] \geq t  \right]. 
    \end{align*}
Now, using uniform convergence bounds (see e.g. \cite[Page 201]{bousquet2003introduction}) for $\mathcal{J}$, which has a VC dimension of $2$ and the fact that for any $J\in \mathcal{J}$, $\Pr\left[ Z_{m}^{\left( i , j \right)} \in J  \right] \leq {w}$, we have that 
\[
\Pr_{\U}\left[  \max_{J\in \mathcal{J}} \sum_{i,j,m} \mathbb{I}\left[Z^{(i,j)}_{m} \in J \right] \geq T \ell k w + 10 \sqrt{T\ell w k \ln(T\ell k / \delta)} + 10\log\left( 10 T \ell k / \delta \right)  \right] \leq \frac \delta 2.
\]
    
        Replacing in values of $k= \frac{\ln(2T\ell/\delta)}{\sigma}$ and using the result of the above coupling, we have     
    \begin{equation*}
        \Pr\left[  \max_{J\in \mathcal{J}} \sum_{i,j} \mathbb{I}  \left[ d_{i,j} \in J \right] \geq  \frac{T \ell w}{\sigma} \log \left( \frac{2 T \ell }{ \delta} \right)  + 10\sqrt{ \frac{T w \ell }{ \sigma} \log \left( \frac{2 T \ell }{ \delta} \right)  \ln \left( \frac{1}{\delta}\right)  } + 10\log\left( \frac{10 T \ell \log \left( 2T \ell /\delta \right) }{\sigma \delta} \right)  \right] \leq \delta
    \end{equation*} 
    as required. 
\end{proof}

\subsection{Proof of \cref{thm:ada_disp}}
The proof of this theorem follows directly from \cref{lem:dispersion-instances} and by setting $w = \sigma (T \ell)^{\alpha-1} $ for $\alpha \geq 0.5$.\qed
 
We note that \cref{thm:ada_disp} shows that even adaptive smooth  adversaries generate sequence of functions that are sufficiently dispersed. This result enables us to directly tap into the results and algorithms of \cite{Dispersion} that show that online optimizing on any dispersed sequence enjoys improved runtime and regret bounds.

%% file: other_related_work.tex
\section{Other Related Work}

In this section, we will survey other work related to the question that we study in this paper. 

\paragraph{Online learning:} 
Similar models of smoothed online learning have been considered in prior work.
For a more thorough discussion, see \cite{haghtalab2020smoothed} and the references therein.  
\cite{NIPS2011_4262} consider online learning when the adversary is constrained in various ways and introduce constrained versions of sequential Rademacher complexity for analyzing the regret.
The work with general setting of sequential symmetrization and tangent sequences introduce in the context of general online learning but adapted to the constrained setting. 
Though these techniques have been applied to other constrained settings \cite{activelearning}, it is not clear how to apply them to our setting. 
\cite{Gupta_Roughgarden} consider smoothed online learning when looking at problems in online algorithm design.  
They prove that while optimizing parameterized greedy heuristics for Maximum Weight Independent Set imposes regret growing linear in $T$ in the worst-case, in presence of smoothing this problem can be learned with non-trivial sublinear regret (as long they allow per-step runtime that grows with $T$).
\cite{Cohen-Addad} consider the same problem with an emphasis on the per-step runtime being logarithmic in $T$. 

{
Smoothed analysis has also been used in a number of other online settings. 
For linear contextual bandits, \cite{kannan2018smoothed} use smoothed analysis to show that the greedy algorithm achieves sublinear regret even though in the worst case it can have linear regret. 
\cite{raghavan2018externalities} work in a Bayesian version of this setting and achieve improved regret bounds for the greedy algorithm. 

{Generally, our work is also related to a line of work on online learning in presence of additional assumptions modelling properties exhibited by real life data. 
\cite{PredictableSequences} consider settings where the learner has additional information available in terms of an estimator for future instances. 
They achieve regret bounds that are in terms of the path length of these estimators and can  beat $\Omega(\sqrt{T})$ if the estimators are accurate. 
\cite{dekel2017online} also considers the importance of incorporating side information in the online learning framework and show that regrets of $O(\log(T))$ in online linear optimization maybe possible when the learner has access to vectors that are weakly correlated with the future instances.}

{
More broadly, our work is among a growing line of work on beyond the worst-case analysis of algorithms~\cite{roughgarden_2020}.
Examples of this in machine learning mostly include improved runtime and approximation guarantees of supervised
(e.g., \cite{LearningSmoothed,kalai2008decision,onebit,Masart}),
and unsupervised settings~(e.g., \cite{bilu_linial_2012, kcenter, stable_clustering, TopicModelling, Decoupling,VDW, MMVMaxCut,llyods, HardtRoth}). 
}

\paragraph{Discrepancy:} Discrepancy is well-studied area in computer science and combinatorics with rich connections to various areas. 
    For a general overview of the area see \cite{chazelle_2000}. 
    Many classical settings such as the Spencer problem, Komlos problem, Tusnandy problem and the Beck-Fiala problem continue to inspire active research. 
    A recent line of work has been developing algorithmic techniques for many new settings that were previously only dealt with  non-constructively and were even believed to be non-tractable \cite{disc1,disc2,disc3,disc4,disc5}.  
    
    A setting that has also recently received attention is the online discrepancy setting. 
    \cite{online_disc1} consider the setting where the inputs are all uniform on $ \left\{ -1,1\right\}^n$ and get a $ O\left( \sqrt{n} \log T \right) $ bound for the $\ell^{\infty}$ discrepancy. 
    Motivated by questions in envy minimization, \cite{bansal2020online} and \cite{online_disc3}, consider the stochastic problem with general distributions, along with several geometric discrepancy problems such as the Tusnady problem. 
    \cite{bansal2020online} gives a $ O\left( n^2 \log T \right) $ discrepancy in the $\ell^{\infty}$ norm algorithm when the input is in $\left[ -1,1 \right]^n $. 
    As discussed earlier, \cite{Bansal_Discrepancy} provide a $ \sqrt{n} \log^4 \left( nT \right)$ in the same setting. 
    They also consider various other settings such as the online Banaszczyk problem and a weighted multicolor discrepancy problem. 
    \cite{ALS_Disc} consider a non-stochastic version of the problem where the vectors are obliviously picked from $\left[ -1,1 \right]^n$ and propose a beautiful randomized algorithm that achieves $ \sqrt{n} \log \left( nT \right) $ bound.

%% file: app_claim_oblivious-bernstein-regret.tex
\section{Uniform Convergence Bounds under Independence}
\label{app:claim:oblivious-bernstein-regret}

    \begin{lemma}[\cite{boucheron2013concentration}]  \label{lem:VC1}  
        Let $\mathcal{A}$ be a countable class of measurable subsets of $\X$ with $\vc\left(  \mathcal{A} \right) =d $ 
        Let $Z_1, \dots Z_n$ be independent random variables taking values in $\X$. Assume that $ \Pr\left[ X_i \in A  \right] \leq \epsilon $ for all $A \in \mathcal{A} $. 
        Let 
        \begin{equation*}
            Q = \frac{1}{\sqrt{n}} \sup_{A \in \mathcal{A} } \sum_{i=1}^n \left(  \mathbb{I} \left[ X_i \in A \right] - \Pr\left[ X_i \in A  \right] \right) . 
        \end{equation*}
        Then, 
        \begin{equation*}
            \Ex\left[ Q \right] \leq 72 \sqrt{ \epsilon d \log \left( \frac{4 e^2}{ \epsilon } \right)  }
        \end{equation*}
        whenever $ \epsilon \geq  \frac{120 d \log\left( \frac{4e^2}{\epsilon} \right) }{n } $. 
    \end{lemma}

    We use the above theorem to get the required bound for the expected maximum of the process indexed by a VC class under our coupling. 

\begin{lemma} \label{lem:VC} 
        Let $\G$ be a class with $\vc \left( \G \right) = d $ and $g \in \G $, $\Ex g( \gamma ) \leq \epsilon$ where $\gamma$ is uniformly distributed. 
        Then, for $ \left\{ \gamma_i \right\}_{i \in \left[ Tk \right] }  $ independetly and uniformly distributed,   
        \begin{equation*}
            \Ex \, \left[ \sup_{g \in \G } \sum_{i} g\left( \gamma_i \right)   \right] \leq 72 \sqrt{\epsilon Tk d \log \left( 1 / \epsilon \right) } + Tk\epsilon
        \end{equation*}
        for $\epsilon > \frac{120d \log \left( 4 e^2 / \epsilon  \right)}{ Tk }   $. 
    \end{lemma}
    \begin{proof}
        Consider the random variable $Q = \frac{1}{ \sqrt{Tk} }  \left[ \sup_{g \in \G} \sum_{i = 1}^{Tk} g\left( \gamma_i \right) - \Ex \left[ g\left( \gamma_i \right) \right] \right] $ where $\gamma_i$ are independent uniform random variables. 
        Note that $ \Ex \left[ g\left( \gamma_i \right) \right] \leq \epsilon $. 
        Note that this satisfies the conditions of \cref{lem:VC1} 
        Thus, 
        \begin{equation*}
            \Ex \left[ Q \right] \leq 72 \sqrt{\epsilon d \log\left( \frac{4 e^2}{\epsilon} \right) },
        \end{equation*}
        whenever $ \epsilon \geq  \frac{120 d \log \left( \frac{4e^2}{\epsilon} \right) }{T k }  $. 
        Thus, we have 
        \begin{equation*}
           \Ex \left[ \sup_{g \in \G} \sum_{i = 1}^{Tk} g\left( \gamma_i \right) - \Ex \left[ g\left( \gamma_i \right) \right] \right] \leq 72 \sqrt{\epsilon T k d \log\left( \frac{4 e^2}{\epsilon} \right) }. 
        \end{equation*}
        Recalling that $\Ex \left[ g\left( \gamma_i \right) \right]\leq \epsilon $, we get the desired result.
    \end{proof}

%% file: Coupling.tex
\section{Coupling Argument } \label{sec:main_coupling}
In this section, we will produce a coupling between a adaptive sequence of $\sigma$-smooth distributions $\adist$ and independent draws from the uniform distribution. 
Initially, we will focus on the case when the $\sigma$-smooth distributions are uniformly distributed on subsets of size $\sigma n$. 

\subsection{Warm-Up: Coupling for a Single Round} \label{sec:single_coupling}

As a warm up, let us look at the coupling for a single smooth distribution. 
Let $ S \subseteq \left[ n \right] $ with $\abs{S } = \sigma n $  and let $X_1 \sim S$. 
Consider the following coupling 
\begin{framed}
    \begin{itemize}
        \item Draw $ k = \alpha \sigma^{-1} $ samples $Y_1 \dots Y_{k}$ from the uniform distribution. 
        \item If  $Y_i \notin S $, then $Z_i = Y_i$. 
        \item Else, for $i$ such that $Y_i \in S $, sample $\tilde{W_i} \sim S  $ and set $Z_i =\tilde{W_i}$. 
        \item Pick $X_1  $ randomly from $ \left\{ \tilde{W_i}   \right\} $. If there is no $Y_i \in S$, then sample $X_1$ uniformly from $S$. 
        \item Output $ \left( X_1 , Z_1 , \dots Z_k \right) $. 
    \end{itemize} 
\end{framed}

In the following lemma, we capture the required properties of the coupling. 

\begin{lemma} \label{lem:single_round}
    Let $\left( X_1 , Z_1 , \dots Z_k \right)$ be as above. 
    Then, 
    \begin{enumerate}
         \item[a.] $X_1$ is uniformly distributed on $S$. 
        \item[b.] $Z_i$ are uniformly distributed on $\left[ n \right]$. 
        \item[c.] Furthermore, $Z_i$ are independent.  
        \item[d.] With probability $1- \left( 1 - \sigma \right)^{\frac{\alpha}{\sigma}} $, $X_1 \in  \{ Z_1, \dots, Z_k   \}$. 
    \end{enumerate}
\end{lemma}
\begin{proof}
    The first part follows from the construction since $X_1$ is either independently drawn from $S$ or set to be equal to some $\tilde{W}_i$ each of which is uniformly distributed on $S$.
    
    For the second part, note that for any $\ell \notin S$, we have 
    $ \Pr\left[ Z_i = \ell  \right] = \Pr\left[ Y_i = \ell \right] = n^{-1} $ and for $ \ell \in S  $, $\Pr\left[ Z_i = \ell \right] = \Pr\left[ Y_i \in S \right] \Pr\left[ \tilde{W}_i = \ell \right] = \sigma n^{-1} \sigma^{-1} = n^{-1} $ as required. 
    In order to argue the independence of $Z_i$, note that $Z_i$ is a function of $Y_i $ and $\tilde{W_i}  $ which are all mutually independent. 

    Note that $X_1 \notin \{ Z_1, \dots, Z_k   \} $ if and only if  for all $i $, $Y_i \notin S $. 
    Since $\abs{S}  = \sigma n $, $\Pr\left[ Y_i \notin S \right] = \left( 1 - \sigma \right)   $ and thus we have 
    \begin{equation*}
        \Pr\left[ X_1 \notin \{ Z_1, \dots, Z_k   \} \right] = \left( 1 - \sigma \right)^k =  \left( 1 - \sigma \right)^{\frac{\alpha}{\sigma}}
    \end{equation*}
    as required. 
\end{proof}

\subsection{Adaptive Coupling} \label{sec:Coupling}

Moving to the case of a sequence of distributions, given a smooth sequence of distribution $\adist$, we would like to find a coupling with a sequence of independent samples from the uniform distribution. 
We first note that an adaptively chosen sequence distribution $\adist $ corresponds to a sequence of sets $S_1, \dots S_n$ of size $n\sigma$ such that $X_i \sim S_i$ where $S_i $ depends on the instantiations of $X_j$ for $j < i$. 
To make this dependence explicit will denote this as $S_i \left( X_1, \dots X_{i-1} \right)$. 
We would like to construct a coupling similar to the one in \cref{sec:single_coupling}. 
Consider the following coupling 

\begin{framed}
    \begin{itemize}
        \item For $j = 1 \dots t$, 
        \begin{itemize}
            \item Draw $ k = \alpha \sigma^{-1} $ samples $Y^{ \left( j \right) }_1 , \dots , Y^{(j)}_{k}$ from the uniform distribution. 
            \item If  $Y^{\left( j \right)}_i \notin S_j \left( X_1 , \dots , X_{j-1} \right) $, then $Z^{\left( j \right)}_i = Y^{\left( j \right)}_i$. 
            \item Else, for $i$ such that $Y^{\left( j \right)}_i \in S_j \left( X_1 , \dots , X_{j-1} \right) $, sample $\tilde{W}^{\left( j \right)}_i \sim S_j \left(  X_1 , \dots , X_{j-1} \right)  $ and set $Z^{\left( j \right)}_i =\tilde{W}_i^{\left( j \right)}$. 
            \item Pick $X_j  $ randomly from $ \left\{ \tilde{W}_i^{\left( j \right)}   \right\} $. If there is no $Y^{\left( j \right)}_i \in S_j \left( X_1 , \dots ,X_{j-1} \right) $, then $X_j \sim S_j \left( X_1 , \dots ,X_{j-1} \right)   $. 
        \end{itemize}
        \item Output $ \left( X_1 ,  Z_{1}^{ \left( 1 \right)} , \dots , Z_{k}^{\left( 1 \right)}   , \dots, X_t , Z_{1}^{ \left( t \right)} , \dots , Z_{k}^{\left( t \right)}  \right)  $. 
    \end{itemize}
\end{framed}

\begin{theorem} \label{thm:main_coupling_1}
    Let $ \left( X_1 ,  Z_{1}^{ \left( 1 \right)} , \dots , Z_{k}^{\left( 1 \right)}   , \dots, X_t , Z_{1}^{ \left( t \right)} , \dots , Z_{k}^{\left( t \right)}  \right)  $ be as above. 
    Then, 
    \begin{itemize}
        \item[a.] $X_1 , \dots , X_t$ is distributed according $\adist$. 
        \item[b.] $ Z_i^{\left( j \right)} $ are uniformly distributed on $\left[ n \right]$. 
        \item[c.] Furthermore, $ Z_i^{\left( j \right)} $ are all mutually independent. 
        \item[d.] With probability at least $1 - t \left( 1 - \sigma \right)^{\frac{\alpha}{\sigma}}  $, $ \left\{ X_1 , \dots , X_t\right\} \subseteq \left\{ Z_i^{\left( j \right)} \right\}_{ i \in \left[ k \right] , j\in \left[ t \right] } $ . 
    \end{itemize}
\end{theorem}
\begin{proof}
    To see that $X_1 \dots X_t$ is distributed according to $\adist$, note that from the construction and \cref{lem:single_round}, we have that conditioned on $X_1 \dots X_{i-1}$, $X_i$ is uniformly distributed according on $S\left( X_1, \dots , X_{i-1} \right)$ as required.   

    In order to show the independence of $ Z_{i}^{\left( j \right)} $, consider any subset $T$ of $ \left\{ Z_{i}^{\left( j \right)} \right\}   $. 
    Let $T =  \{ Z_{i_1}^{\left( j_1 \right)} , \dots , Z_{i_m}^{\left( j_m \right)}  \} $, with $j_1 \leq j_2 \dots \leq j_m $. 
    Then, 
    \begin{align*}
        \Pr \left[ Z_{i_m}^{\left( j_m \right)} = z_{i_m}^{\left( j_m \right)} | \left\{ Z_{i_p}^{\left( j_p \right)} = z_{i_p}^{\left( j_p \right)} \right\}_{p < m}  \right]  = \Pr \left[  z_{i_m}^{\left( j_m \right)} \notin S\left( X_1 \dots X_{j_m - 1} \right)  \land Y_{i_m}^{j_m} = z_{i_m}^{\left( j_m \right)} \bigg\vert \left\{ Z_{i_p}^{\left( j_p \right)} = z_{i_p}^{\left( j_p \right)} \right\}_{p < m}  \right]  \\
         + \Pr \left[  z_{i_m}^{\left( j_m \right)} \in S\left( X_1 \dots X_{j_m - 1} \right)  \land Y_{i_m}^{j_m} \in S\left( X_1 \dots X_{j_m - 1} \right) \land \tilde{W}_{i_m}^{\left( j_m \right)} = z_{i_m}^{\left( j_m \right)} \bigg\vert \left\{ Z_{i_p}^{\left( j_p \right)} = z_{i_p}^{\left( j_p \right)} \right\}_{p < m}  \right] . 
    \end{align*}
    We will deal with each of these terms separately. 
    \begin{align*}
        \Pr \left[  z_{i_m}^{\left( j_m \right)} \notin S\left( X_1 \dots X_{j_m - 1} \right)  \land Y_{i_m}^{j_m} = z_{i_m}^{\left( j_m \right)} \bigg\vert \left\{ Z_{i_p}^{\left( j_p \right)} = z_{i_p}^{\left( j_p \right)} \right\}_{p < m}  \right] \\ = \frac{1}{n} \cdot \Pr \left[  z_{i_m}^{\left( j_m \right)} \notin S\left( X_1 \dots X_{j_m - 1} \right) \bigg\vert \left\{ Z_{i_p}^{\left( j_p \right)} = z_{i_p}^{\left( j_p \right)} \right\}_{p < m}  \right] . 
    \end{align*}
    This follows since $Y_{i_m}^{j_m}  $ is independent of $X_1 \dots X_{j_m - 1}$ and $\left\{ Z_{i_p}^{\left( j_p \right)}  \right\}_{p < m} $. 
    Moving to the second term, 
    \begin{align*}
        &\Pr \left[  z_{i_m}^{\left( j_m \right)} \in S\left( X_1 \dots X_{j_m - 1} \right)  \land Y_{i_m}^{j_m} \in S\left( X_1 \dots X_{j_m - 1} \right) \land \tilde{W}_{i_m}^{\left( j_m \right)} = z_{i_m}^{\left( j_m \right)} \bigg\vert \left\{ Z_{i_p}^{\left( j_p \right)} = z_{i_p}^{\left( j_p \right)} \right\}_{p < m}  \right] = \\
        &\hspace*{ - 55 pt}\Ex \left[ \Pr \left[  z_{i_m}^{\left( j_m \right)} \in S\left( x_1 \dots x_{j_m - 1} \right)  \land Y_{i_m}^{j_m} \in S\left( x_1 \dots x_{j_m - 1} \right) \land \tilde{W}_{i_m}^{\left( j_m \right)} = z_{i_m}^{\left( j_m \right)} \bigg\vert \left\{ Z_{i_p}^{\left( j_p \right)} = z_{i_p}^{\left( j_p \right)} \right\}_{p < m} , X_1 = x_i \dots X_{j_m - 1} = x_{j_m - 1}  \right]\right] = \\ 
        & \frac{1}{n} \cdot \Ex \left[ \Pr \left[  z_{i_m}^{\left( j_m \right)} \in S\left( x_1 \dots x_{j_m - 1} \right) | \left\{ Z_{i_p}^{\left( j_p \right)} = z_{i_p}^{\left( j_p \right)} \right\}_{p < m} , X_1 = x_i \dots X_{j_m - 1} = x_{j_m - 1}  \right]  \right] = \\ 
        &  \frac{1}{n} \cdot \Pr \left[  z_{i_m}^{\left( j_m \right)} \in S\left( X_1 \dots X_{j_m - 1} \right) \bigg\vert \left\{ Z_{i_p}^{\left( j_p \right)} = z_{i_p}^{\left( j_p \right)} \right\}_{p < m}  \right] . 
    \end{align*}
    Summing the two terms, we get 
    \begin{align*}
        \Pr \left[ Z_{i_m}^{\left( j_m \right)} = z_{i_m}^{\left( j_m \right)} | \left\{ Z_{i_p}^{\left( j_p \right)} = z_{i_p}^{\left( j_p \right)} \right\}_{p < m}  \right] = \frac{1}{n} 
    \end{align*}
    Recursively applying this to $T \setminus \{ Z_{i_m}^{\left( j_m \right)} \} $, we get 
    \begin{equation*}
        \Pr \left[ \left\{ Z_{i_p}^{\left( j_p \right)} = z_{i_p}^{\left( j_p \right)} \right\}_T \right] = \frac{1}{n^{\abs{T}}}
    \end{equation*}
    proving the required independence. 

    As in \cref{lem:single_round}, we have that the probability that $X_j \notin \{ Z_i^{\left( j \right)} \} $ is bounded by $ \left( 1 - \sigma \right)^{\frac{\alpha}{\sigma}} $. 
    By the union bound, we have 
    \begin{equation*}
        \Pr \left[ \exists j : X_j \notin \{ Z_i^{\left( j \right)} \} \right] \leq t \cdot \left( 1 - \sigma \right)^{\frac{\alpha}{\sigma}} 
    \end{equation*} 
    as required. 
\end{proof}

\subsection{General Smooth Distributions}
In order to move from the special case of uniform distributions on subsets of size $n \sigma$, we need to note that smooth distributions are convex combinations of uniform distributions on subsets of size $\sigma n $. 

\begin{lemma}
    Let $\mathcal{P}$ be the set of $\sigma$ smooth distributions on $\left[ n \right]$ and let $\mathcal{P}_0$ be the set of distributions that are uniform on subsets of size $\sigma n $. 
    Then, 
    \begin{equation*}
        \mathcal{P} = \mathrm{conv} \left( \mathcal{P}_0 \right). 
    \end{equation*}
\end{lemma}

In particular, this implies that for each $ \sigma  $-smooth distribution $\mathcal{D}$, there is a distribution $\mathcal{S}_{\mathcal{D}} $ on subsets of size $\sigma n$ such that sampling from $\mathcal{D} $ can be achieved by first sampling $S \sim \mathcal{S}_{\mathcal{D}} $ and then sampling uniformly from $S$. 

Moving onto adaptive sequences of smooth distributions, recall that corresponding to each $X_1 , \dots , X_i$ is a smooth distribution $ \adist_t \left( X_1, \dots , X_i  \right) $. 
We will use $\mathcal{S}_t \left( X_1 , \dots , X_i \right) $ to denote $\mathcal{S}_{ \adist\left( X_1 , \dots , X_i \right) }$. 
The idea is to use this in conjunction with the coupling from discussed earlier to get a coupling for all $\sigma$-smooth distributions.
That is, in each stage, a set $S_t $ is first sampled from $ \mathcal{S}_t \left( X_1 , \dots , X_i \right)  $ and then $S_t$ is used in the previous mentioned coupling. 
For infinite domains, similar argument can be made using the Choquet integral representation theorem which gives a way to represent smooth distributions as convex combinations of uniform distributions on sets of large measure. 
Putting this together leads to \cref{thm:main_coupling_overview}. 

\begin{framed}
    \begin{itemize}
        \item For $j = 1 \dots t$, 
        \begin{itemize}
            \item Sample $ k = \alpha \sigma^{-1}  $ many samples from the uniform distribution.
            \item  Let $S_j \left( X_1 , \dots , X_{j-1} \right)$ be sampled from $\mathcal{S }_{j} \left( X_1 , \dots , X_{j-1} \right) $. 
            \item If  $Y^{\left( j \right)}_i \notin S_j \left( X_1 , \dots , X_{j-1} \right) $, then $Z^{\left( j \right)}_i = Y^{\left( j \right)}_i$. 
            \item Else, for $i$ such that $Y^{\left( j \right)}_i \in S_j \left( X_1 , \dots , X_{j-1} \right) $, sample $\tilde{W}^{\left( j \right)}_i \sim S_j \left(  X_1 , \dots , X_{j-1} \right)  $ and set $Z^{\left( j \right)}_i =\tilde{W}_i^{\left( j \right)}$. 
            \item Pick $X_j  $ randomly from $ \left\{ \tilde{W}_i^{\left( j \right)}   \right\} $. If there is no $Y^{\left( j \right)}_i \in S_j \left( X_1 , \dots ,X_{j-1} \right) $, then $X_j \sim S_j \left( X_1 , \dots ,X_{j-1} \right)   $. 
        \end{itemize} 
        \item Output $ \left( X_1 ,  Z_{1}^{ \left( 1 \right)} , \dots , Z_{k}^{\left( 1 \right)}   , \dots, X_t , Z_{1}^{ \left( t \right)} , \dots , Z_{k}^{\left( t \right)}  \right)  $. 
    \end{itemize}
\end{framed}

\begin{theorem}[\cref{thm:main_coupling_overview} restated] \label{thm:main_coupling}
    Let $ \adist $ be an adaptive sequence of $\sigma$-smooth distribution on $\X$. 
    Then, there is a coupling $\Pi$ such that $ \left( X_1 ,  Z_{1}^{ \left( 1 \right)} , \dots , Z_{k}^{\left( 1 \right)}   , \dots, X_t , Z_{1}^{ \left( t \right)} , \dots , Z_{k}^{\left( t \right)}  \right) \sim \Pi $ satisfy
    \begin{itemize}
        \item[a.] $X_1 , \dots , X_t$ is distributed according $\adist$. 
        \item[b.] $ Z_i^{\left( j \right)} $ are uniformly and independently distributed on $\X$. 
        \item[c.] $ \{ Z_i \left( j \right) \}_{j \geq t} $ are independent and uniform conditioned on $X_1 , \dots, X_{t-1}$.  
        \item[d.] With probability at least $1 - t \left( 1 - \sigma \right)^{\frac{\alpha}{\sigma}}  $, $ \left\{ X_1 , \dots , X_t\right\} \subseteq \left\{ Z_i^{\left( j \right)} \mid i \in \left[ k \right] , j\in \left[ t \right] \right\} $ . 
    \end{itemize}
\end{theorem}

%% file: LDProof.tex
   \section{Proofs from \cref{sec:RegretBounds}}
   
   \begin{lemma} \label{lem:Littlestone}
       Let $ \H $ be the class defined on $ \left[ \nicefrac{1}{\sigma} \right] $ as the disjoint union of $ d $ thresholds as in \cref{sec:regretlowerproof}. 
       Then, the Littlestone dimension of $ \H$ is lower bounded by $ \Omega\left( \sqrt{ d \log \left( \nicefrac{1}{d\sigma} \right)  } \right)  $. 
   \end{lemma}
   \begin{proof}
    In order to prove this associate to each string $ \left\{ 0, 1 \right\}^{d \log\left( 1 / \sigma d \right)} $ a function in $ \mathcal{H}$ as follows. 
    Partition the string into blocks of size $ \frac{1}{\sigma d}$. 
    We think of each of these blocks as forming a binary search tree for the subset $A_i$ by associating $ 1 $ to the right child of a node and $0$ to the left child. 
    Thus, every path on this tree corresponds to a threshold by associating it with the threshold consistent with the labels along the path.  
    Doing this association separately for each block, we can associate the set of strings $ \left\{ 0, 1 \right\}^{d \log\left( 1 / \sigma d \right)} $ with a binary search tree with the leaves labeled by elements in $ \H$. 
    Also, note that this forms a fully shattered tree as required by the definition of the Littlestone dimension. 
    Thus, the Littlestone dimension of $ \H$ is $ d \log \left( \nicefrac{1}{\sigma d} \right) $. 
   \end{proof}

%% file: app_discrepanc_proofs.tex
    \section{Proofs from \cref{sec:Discrepancy} }  \label{sec:miss_proof_disc}

    \begin{lemma} [\cite{Bansal_Discrepancy}]
        \begin{equation*}
            \Ex_{X_t} \left[ Q \left( X_t \right)  \right] \leq c \lambda^2  \Ex_{ W \sim p }  \abs{\sinh \left( \lambda d_{t - 1}^{\top }  W   \right)} 
        \end{equation*}
        and 
        \begin{equation*}
            \Ex_{X_t} \left[ Q_{*} \left( X_t \right)  \right] \leq \frac{c \lambda^2}{n }
        \end{equation*}
    \end{lemma}
    \begin{proof}
        \begin{align*}
            \Ex_{X_t} \left[ Q \left( X_t \right)   \right] &  =  \Ex_{X_t} \left[   \lambda^2  \Ex_{ W \sim p }  \abs{\sinh \left( \lambda d_{t - 1}^{\top }  W   \right)} W^{\top} X_tX_t^{\top}  W \right] \\
            & =    \lambda^2  \Ex_{ W \sim p }  \abs{\sinh \left( \lambda d_{t - 1}^{\top }  W   \right)} W^{\top} \Ex_{X_t} \left[  X_tX_t^{\top}  \right]  W \\  
            & =  c \lambda^2     \Ex_{ W \sim p }  \abs{\sinh \left( \lambda d_{t - 1}^{\top }  W   \right)} 
        \end{align*}
        Similarly, 
        \begin{align*}
            \Ex_{X_t }  \left[ Q_{*} \left( X_t \right)   \right] &=  \Ex_{X_t } \left[ \lambda^2  \Ex_{ W \sim p }   W_j^{\top}  X_t X_t^{\top} W \right] \\ 
            & \leq \frac{c}{n}\lambda^2
        \end{align*}
        as required. 
    \end{proof}